\documentclass[twoside,11pt]{article}

\newcommand{\change}[2]{}
\newcommand{\lchange}[2]{}

\newcommand{\changed}[3]{#3}

\usepackage{jmlr2e}
\usepackage{marginnote}
\usepackage{listings}
\usepackage{xcolor}

\usepackage{times,subfigure}
\usepackage{epsfig}
\usepackage{graphicx}
\usepackage{pifont}
\usepackage{amssymb,bm}
\usepackage{amsfonts}
\usepackage{booktabs}
\usepackage{threeparttable}
\usepackage{appendix}
\usepackage[ruled,linesnumbered]{algorithm2e}

\usepackage{amsmath}

\usepackage[capitalise]{cleveref}

\usepackage{color}
\usepackage{lineno}
\usepackage{multirow}
\usepackage{makecell}
\usepackage{microtype}
\usepackage{url}            
\usepackage{nicefrac}       
\usepackage{mathrsfs}
\usepackage{pgflibraryarrows}
\usepackage{pgflibrarysnakes}
\usepackage{tikz}
\usepackage{pgfplots}

\DeclareMathOperator*{\argmin}{argmin}

\newtheorem{assumption}{Assumption}

\definecolor{fhcolor}{rgb}{0.523, 0.235, 0.625}

\usepackage{lastpage}

\begin{document}

\jmlrheading{25}{2024}{1-\pageref{LastPage}}{11/22; Revised 10/23}{4/24}{22-1250}{Fanghui Liu, Leello Dadi, Volkan Cevher}
\ShortHeadings{Learning with Norm Constrained, Over-parameterized, Two-layer Neural Networks}{Liu, Dadi, and Cevher}
\firstpageno{1}

\title{Learning with Norm Constrained, Over-parameterized, Two-layer Neural Networks}

\author{\name Fanghui Liu\thanks{Part of this work was done when Fanghui was at LIONS, EPFL. Corresponding author: Fanghui Liu.} \email fanghui.liu@warwick.ac.uk \\
\addr Department of Computer Science, University of Warwick, Coventry, UK  \\
\name Leello Dadi \email leello.dadi@epfl.ch\\
\addr Lab for Information and Inference Systems, \'{E}cole Polytechnique F\'{e}d\'{e}rale de Lausanne (EPFL), Switzerland \\
\name Volkan Cevher \email volkan.cevher@epfl.ch \\
\addr Lab for Information and Inference Systems, \'{E}cole Polytechnique F\'{e}d\'{e}rale de Lausanne (EPFL), Switzerland \\
}

\editor{}

\maketitle

\begin{abstract}

\noindent 
Recent studies show that a reproducing kernel Hilbert space (RKHS) is not a suitable space to model functions by neural networks as the curse of dimensionality (CoD) cannot be evaded when trying to approximate even a single ReLU neuron \citep{bach2017breaking,celentano2021minimum}. In this paper, we study a suitable function space for over-parameterized two-layer neural networks with bounded norms (e.g., the path norm, the Barron norm) in the perspective of sample complexity and generalization properties. First, we show that the path norm (as well as the Barron norm) is able to obtain width-independence sample complexity bounds, which allows for uniform convergence guarantees. Based on this result, we derive the improved result of metric entropy for $\epsilon$-covering up to $\mathcal{O}(\epsilon^{-\frac{2d}{d+2}})$ ($d$ is the input dimension and the depending constant is at most linear order of $d$) via the convex hull technique, which demonstrates the separation with kernel methods with $\Omega(\epsilon^{-d})$ to learn the target function in a Barron space.
Second, this metric entropy result allows for building a sharper generalization bound under a general moment hypothesis setting, achieving the rate at $\mathcal{O}(n^{-\frac{d+2}{2d+2}})$.
Our analysis is novel in that it offers a sharper and refined estimation for metric entropy with a linear dimension dependence and unbounded sampling in the estimation of the sample error and the output error.
\end{abstract}


	\section{Introduction}
	\vspace{-0.2cm}
While the number of neurons provide a natural complexity or \textit{capacity} measure for neural networks, recent theoretical approaches focus more on the magnitude of the weights via norm-based constraints \citep{neyshabur2015norm,savarese2019infinite,ongie2020function,domingo2022tighter}.
The norm-based capacity analyses (e.g., even in infinite-width) represent essentially potential functions with small cost.
Generally they revolve around minimum-norm \emph{over-parameterized} models and seek to identify correlations between the norm-based capacity and the generalization of neural network models, e.g., min-$\ell_2$-norm \citep{hastie2019surprises,mei2019generalization,liang2020multiple}, and min-$\ell_1$-norm \citep{liang2022PreciseHighdimensional,chatterji2021foolish,wang2021tight}.

From a functional perspective, one key issue corresponds to what norm can be defined and controlled on the functions defined by neural networks, and what suitable function space is for learning via norm capacity based neural networks.
Indeed, the prototypical two-layer neural networks with the rectified linear unit (ReLU) activation cannot be sufficiently captured by the reproducing kernel Hilbert spaces (RKHSs) of kernel methods such as the random features (RF) model \citep{rahimi2007random} and the neural tangent kernel (NTK) \citep{jacot2018neural}.
Specifically,~to approximate a single ReLU neuron with an $\varepsilon$-approximation error, kernel methods require a number of samples $\Omega(\varepsilon^{-d})$, exponential in feature dimensionality $d$ \citep{bach2017breaking,yehudai2019power,ghorbani2019linearized,wu2022spectral}, \emph{a.k.a.}, curse of dimensionality (CoD).

It is possible to characterize over-parameterized two-layer neural networks by a compactly supported Radon probability measure with spectral Barron norm \citep{barron1993universal}, path norm \citep{neyshabur2015norm}, bounded total variation norm \citep{bach2017breaking} or other variants \citep{savarese2019infinite,ongie2020function,parhi2021banach}.
In this context, any function that is well-approximated by a norm-bounded two-layer neural networks lives in a revised Barron space \citep{weinan2021barron}, which is the \emph{largest}\footnote{The terminology ``largest'' means functions in the Barron space and two-layer neural networks with bounded norm can be efficiently represented without the curse of dimensionality, refer to \cref{sec:related} for details.} function space for two-layer neural networks \citep{weinan2020representation} beyond RKHS.
In this case, given $n$ iid training data, the minimax lower bound in excess risk for learning in this non-Hilbertian space by any kernel methods estimators is $\Omega(n^{- 1/d})$ \citep{weinan2021kolmogorov,parhi2021near}, which suffers from CoD and unsatisfactory sample complexity bounds.

Interestingly, the Barron space can be closely connected to other typical spaces with various norms. 
For example, when using the ReLU activation function, the Barron norm is exactly the same as the total variation norm \citep{bach2017breaking}, and other total variation norm based versions \citep{siegel2021sharp,parhi2021banach}.
Furthermore, the discrete version of the Barron norm is the $\ell_1$-path norm \citep{neyshabur2015norm} in the parameter space, and thus is widely used in practice.
Based on this, understanding on the Barron space, especially based on the path-norm studied in this paper, for learning such two-layer neural networks is general and required.

\subsection{Problem setting}

Mathematically, let $X \subseteq \mathbb{R}^d$ be a compact space and the label space $Y \subseteq \mathbb{R}$, we assume that a sample set $\bm z = \{  (\bm x_i, y_i) \}_{i=1}^n \in Z^n $ is independently drawn from a non-degenerate Borel probability measure $\rho$ on $X \times Y$.
Our target is to find a hypothesis $f$ over the sample set such that $f$ is a good approximation of the \emph{target function}
\begin{equation*}
    f_{\rho}(\bm x) = \int_Y y \mathrm{d} \rho(y|\bm x), \bm x \in X\,,
\end{equation*}
where $\rho(\cdot| \bm x)$ is the conditional distribution of $\rho$ at $\bm x \in X$. 
Here the hypothesis is considered as the class of two-layer neural networks parameterized by $\bm \theta := \{ (a_i, \bm w_i) \}_{i=1}^m$ with $\bm w_i \in \mathbb{R}^d$, $a_i \in \mathbb{R}$
\begin{equation}\label{eq:function}
    \mathcal{P}_{m} = \left\{ f_{\bm \theta}(\bm \cdot) := \frac{1}{m}\sum_{k=1}^m a_k \sigma \big( \langle \bm w_k, \bm \cdot \rangle \big) \right\}\,,
\end{equation}
where $\sigma: \mathbb{R} \rightarrow \mathbb{R}$ is the activation function. It allows for the bias term by adding one extra term but we omit it here for simplicity. The hypothesis is often learned from  the sample set via the empirical risk minimization (ERM) with a proper regularization term
\begin{equation}\label{fz1}
	f_{\bm \theta,\bm z, \lambda} = \argmin_{f \in \mathcal{P}_m} \frac{1}{n} \sum_{i=1}^n \ell(y_i, f(\bm x_i,\bm \theta) ) + \lambda \| f \|\,,
\end{equation}
where $\ell: \mathbb{R} \times \mathbb{R} \rightarrow \mathbb{R}^+$ is a \emph{surrogate} loss function, and $\lambda \equiv \lambda(n) > 0$ is a regularization parameter of a norm-based regularizer $\| f \|$.
We will discuss a certain $\ell_1$-path norm regularization later in Section~\ref{sec:probsetting}, which is used in this paper.
To facilitate a general analysis on learning with noisy data in practice, we consider the following moment hypothesis concerning unbounded outputs.

\emph{Moment hypothesis:} There exist constants $M \geqslant 1$, $C > 0$ such that\footnote{This assumption is essentially equivalent to Bernstein's condition \citep{Steinwart2008SVM}:	$ \mathbb{E}\left[|y|^{b} \mid \bm x\right] \leqslant \frac{1}{2}b ! \varsigma^{2} B^{b-2} $ when $b \geqslant 2$\,. For description simplicity, we assume $y$ is of zero mean.}
	\begin{equation}\label{Momenthypothesis}
		\int_Y |y|^p \emph{d} \rho(y|\bm x) \leqslant C p ! M^p, \forall p
		\in \mathbb{N}, \quad  \bm x \in X\,.
	\end{equation}
Compared to the standard uniform boundedness assumption with $|y| \leq M$ almost surely, this assumption is general since it covers Gaussian noise, sub-Gaussian noise, sub-exponential noise, etc.

Based on the above problem setting, we are interested in sample complexity for uniform estimates and generalization guarantees (these two questions are connected to each other).
The sample complexity issue stems from the following question:
A width-independent sample complexity bound can be obtained by 
Frobenius norm control \citep{neyshabur2015norm}.
However spectral norm control \citep{bartlett2017spectrally} is generally insufficient for two-layer neural networks as demonstrated by \citet{vardi2022sample}.
This conclusion is \emph{questionable} when the Barron norm (as well as the path norm) is employed since the Barron space is larger than that of Frobenius norm-constrained two-layer neural networks.
In this work, we derive the width-independent sample complexity for uniform estimate in the Barron space, and then study the generalization properties under a general setting.
This contributes to provide an in-depth theoretical understanding and refined analysis of learning with over-parameterized two-layer neural networks under norm-based capacity under this general setting.

\subsection{Contributions and technical challenges}

Our analysis starts with the width-independence sample complexity, study the metric entropy, and provide refined results on generalization bounds under weaker conditions. We make the following contributions:

\begin{itemize}
    \item based on the Gaussian complexity metric, we prove that the path norm is able to obtain width-independence sample complexity bounds.
This result motivates us to derive that the metric entropy for $\epsilon$-covering up to $\mathcal{O}(\epsilon^{-\frac{2d}{d+2}})$ via the convex hull technique using the smoothness structure of $\mathcal{P}_m$.
Note that the linear dependence on the input dimension $d$ can be clearly derived, while this is unclear (would be exponential order of $d$) in \citep{siegel2021sharp} based on the orthogonal function argument.
Our estimation provides a good trade-off on $\epsilon$ and $d$.
\item when applying our metric entropy result for generalization guarantees, we are faced with how to tackle the output error in Eq.~\eqref{Momenthypothesis}.
Tackling such unbounded sampling requires the number of training data $n$ to be in an exponential order in previous work \citep{Wang2011Optimal,guo2013concentration,liu2021gen}, i.e., $n \geqslant k^k$ for a sufficiently large $k$, which implies that a sample complexity suffers from CoD due to $k \geqslant d$.
In this work, we introduce new concentration inequalities via sub-Weibull random variables, remove the dependence on the exponential order to bound the output error, and provide non-asymptotic error bounds on finite $n$.
This might be of interest in its own right in the approximation theory community.
\end{itemize}

Combining the improved estimation of metric entropy and the output error, we are ready to present generalization bounds of learning $\ell_1$-path norm based, over-parameterized, two-layer ReLU neural networks. To be specific, a sharper rate $\mathcal{O}(n^{-\frac{d+2}{2d+2}})$ of generalization bounds is derived under a more general setting, which is better than previous work \citep{bach2017breaking,weinan2019priori,wang2021harmless} with $\mathcal{O}(\sqrt{\log n/n})$.\footnote{We remark that the optimal approximation rate of the Barron space is given by \citep{siegel2021sharp} and \citep{wu2022spectral} based on different techniques, and the minimax rate at $\mathcal{O}(n^{-\frac{d+3}{2d+3}})$ can be derived under the variation norm based space, e.g., \citep{parhi2021near} with the skip connection. However, the dependence on $d$ is still unclear. We will detail this discussion in \cref{sec:excessrisk}.}
This requires a refined analysis in our proof for the sample error and the output error. 
Besides, optimization over the Barron space is normally NP-hard.
We attempt to develop a computational (but not efficient except for the low-rank data) algorithm based on the measure representation and convex duality \citep{chizat2021convergence,pilanci2020neural}. The basic over-parameterization condition ($m \geqslant n+1$) is sufficient.

Our results are immune to CoD even though $d$ is large enough; while kernel methods are not.
This demonstrates the separation between over-parameterized neural networks and kernel methods when learning in the Barron space.
We hope that our analysis has a better understanding on the improved analysis of neural networks in a norm-based capacity view and function spaces in the machine learning community.

\subsection{Organization and notations}

The rest of the paper is organized as follows. In Section~\ref{sec:related}, we give an overview of related work and preliminaries on two-layer neural networks in a norm-based capacity view.
In Section~\ref{sec:probsetting}, we consider a general random design regression problem for two-layer neural networks with bounded norm in the context of statistical learning theory.
The assumptions and our main results are stated in Section~\ref{sec:mainresult}. 
We outline the proof framework in Section~\ref{sec:proofframe}.
We validate our findings with numerical simulations in Section~\ref{sec:exp}.
The conclusion is drawn in Section~\ref{sec:conclusion}. 
Some technical lemmas and proofs are deferred to the appendix.

For notations, we use the shorthand $ [n]:= \{1,2,\dots, n \}$ for some positive $n$. 
The unit sphere of $\mathbb{R}^d$ is defined as
$\mathbb{S}_p^{d-1} = \{\bm x \in \mathbb{R}^d \big| \| \bm x \|_p = 1 \}$.
Let $\delta_{\bm w}(\cdot)$ be the Dirac measure on $\bm w \in \mathbb{R}^d$, i.e., $\int f(\bm x) \delta_{\bm w} (\mathrm{d} \bm x) = f(\bm w)$.
The notation $a(n) \lesssim b(n)$ signifies that there exists a positive constant $c$ independent of $n$ such that $a(n) \leqslant c b(n)$.
We use the standard big-O notation with $\mathcal{O}(\cdot)$, $\Omega(\cdot)$ hiding constants and $\widetilde{\mathcal{O}}(\cdot)$ hiding constants and factors polylogarithmic in the problem parameters.
Besides, as noted in \citep{weinan2020representation}, 
if $m$ tends to infinity, $\mathcal{P}_m$ is a closed subspace of a Barron space, and thus is still a Banach space.
For any compact $X \subset \mathbb{R}^d$, the $\textsf{C}(X)$-closure of $\mathcal{F}_{\infty} := \cup_{m \in \mathbb{N}} \mathcal{F}_m$ is the entire space of $\textsf{C}(X)$, where $\textsf{C}(X)$ is the function space of all continuous function on $X$ with $\| \cdot \|_{\infty}$.

\section{Related Works and Preliminaries}
\label{sec:related}
\vspace{-0.cm}

We give an overview of two-layer neural networks via integral representation in the Barron space \citep{weinan2021barron} and other function spaces with various norms.

\noindent {\bf Two-layer neural networks and integral representation:}
We consider a two-layer neural network with $m$ neurons represented as $f(\bm x) \!=\! \frac1m \! \sum_{k=1}^m \! a_k\sigma(\bm w_k^{\!\top} \! \bm x)$,
where $\sigma(\cdot)$ is the activation function and the parameters (weights) of the network are $\{ (a_k, \bm w_k) \}_{k=1}^m \subset \mathbb{R} \times \mathbb{R}^d$.
This setting allows for neural networks with bias, by taking $(\bm x,1) \in \mathbb{R}^{d+1}$ as $\bm x$ and $(\bm w, b) \in \mathbb{R}^{d+1}$ as $\bm w$.
For notational simplicity, we still use $\bm x \in \mathbb{R}^d$ in this paper.
We consider the above two-layer neural network in a general integral representation
\begin{equation*}
f(\bm x)=\int_{\Omega} a \sigma (\bm w^{\!\top} \bm x ) \tilde{\mu}(\mathrm{d} a, \mathrm{d} \bm w), \quad \bm x \in X\,,
\end{equation*}
where $\Omega = \mathbb{R} \times \mathbb{R}^d$ and $\tilde{\mu}$ is a probability measure over $(\Omega, \mathcal{T}(\Omega))$ with $\mathcal{T}(\Omega) $ being a Borel $\sigma$-algebra on $\Omega$. 
The used activation function in this work is ReLU $\sigma(x) = \max(0,x)$.

Due to the scaling invariance of ReLU, we can assume $\bm w \in \mathbb{S}^{d-1}$ such that
\begin{equation}\label{eq:frep}
f(\bm x)=\int_{\mathbb{S}^{d}} a \sigma\left(\bm w^{\!\top} \bm x\right) \tilde{\mu}(\mathrm{d} a, \mathrm{d} \bm w)=\int_{\mathbb{S}^{d}} a(\bm w) \sigma\left(\bm w^{\!\top} \bm x\right) \mu(\mathrm{d} \bm w)\,,
\end{equation}
where $a(\bm w) = \frac{\int_{\mathbb{R}} a \tilde{\mu}(a, \bm w) \mathrm{d} a}{\pi(\bm w)}$ and $\mu(\bm w)=\int_{\mathbb{R}} \tilde{\mu}(a, \bm w) \mathrm{d} a$.
Accordingly, the representation of $f$ in \cref{eq:frep} can be considered into the associated function space with various functional norms
\begin{equation}\label{fpspace}
	\mathcal{F}_p (R) := \left\{ f(\bm \cdot) = \int_{\mathcal{V}} a(\bm w) \sigma( \langle \bm w, \bm \cdot \rangle ) \mu(\mathrm{d} \bm w) \Big| \| f \|_{\mathcal{F}_p}:= \left( \int_{\mathcal{V}} |a(\bm w)|^p \mu(\mathrm{d} \bm w) \right)^{1/p} \leqslant R \right\} \,. 
\end{equation}
If we take $p=2$, the weights $\bm w$ are defined on a probability space $(\mathcal{V}, \mu)$ with $a \in L^2(\mathcal{V}, \mu)$ and $\mathcal{V}= \mathbb{R}^d$.\footnote{This is equivalent to $\mathcal{V} = \mathbb{S}_p^{d-1}$ due to scale invariance of ReLU.}
In this case, $\mathcal{F}_2$ is a RKHS with the associated reproducing kernel $k(\bm x, \bm x'):= \int_{\mathcal{V}} \sigma(\bm w^{\!\top} {\bm x}) \sigma(\bm w^{\!\top} {\bm x'}) \mu(\mathrm{d} \bm w)$ \citep{bach2017breaking}.
Although $\mathcal{F}_2(R)$ is infinite-dimensional, the RKHS-norm regularized empirical risk minimization (ERM) can be solved by a finite dimensional optimization problem by the representer theorem \citep{scholkopf2018learning}.
Note that, for $p \in [1,2)$ this comprises a rich function class than the original RKHS $\mathcal{F}_2(R)$ due to  $\mathcal{F}_2(R) \subset \mathcal{F}_p(R) \subset \mathcal{F}_1(R)$ \citep{celentano2021minimum}.
The special case is $p=1$ in which $\mu$ is a signed Radon measure on $\mathcal{V}$ with finite total variation $|\mu|(\mathcal{V})$ \citep{bach2017breaking}, i.e., ``convex neural network'' \citep{bengio2005convex}.
This corresponds to a variation norm rather than a RKHS norm, which actually enlarges the space by adding non-smooth functions.

\noindent {\bf Barron spaces under the ReLU activation function:}
Apart from the typical $\mathcal{F}_p$ function space, the space of two-layer ReLU neural networks can be built on Fourier transform.
By taking $\hat{f}$ as the Fourier transform of an extension of $f$ to $\mathbb{R}^d$, the classical \emph{spectral Barron norm} \citep{barron1993universal} is given by $\| f \| := \int_{\mathbb{R}^d} \| \bm \omega \| |\hat{f}(\bm \omega) | \mathrm{d} \bm \omega $.
Its variants include $\| f \| := \inf_{\hat{f}} \int_{\mathbb{R}^d} |\hat{f}(\bm \omega) | \| \bm \omega \|_1^s \mathrm{d} \bm \omega $ \citep{klusowski2016risk} and $\| f \| :=  \int_{\mathbb{R}^d} |\hat{f}(\bm \omega)| [1 +  \| \bm \omega \|_1^2] \mathrm{d} \bm \omega $ \citep{klusowski2018approximation}.
If we replace $|\hat{f}(\bm \omega)|$ with $|\hat{f}(\bm \omega)|^2$, we would obtain Sobolev semi-norm. 

\cite{weinan2021barron} provides a probabilistic interpretation of Barron space as an infinite union of a family of RKHS. The related Barron norm is defined by
\begin{equation*}
    \| f \|_{\mathcal{B}_p} = \inf_{\tilde{\mu}} \left( \mathbb{E}_{\tilde{\mu}} |a|^p \| \bm w \|_1^p \right)^{1/p}\,,
\end{equation*}
where the infimum is taken over all possible $\tilde{\mu}$. 
Particularly, for the ReLU, we have $\mathcal{B}_p = \mathcal{B}_{\infty}$ and $\| f \|_{\mathcal{B}_{p}} = \| f \|_{\mathcal{B}_{\infty}}$ for any $1 \leqslant p \leqslant \infty$, and thus we can directly use $\mathcal{B}$ and $\| f \|_{\mathcal{B}}$ to denote the Barron space and the Barron norm. It closely relates to the $\ell_1$-path norm $\| \cdot \|_{\mathcal{P}}$
\changed{A1.9}{\link{R1.9}}{
\begin{equation}\label{bpnorm}
	\| f_{\bm \theta} \|_{\mathcal{B}} \leqslant	\| \bm \theta \|_{\mathcal{P}} := \frac{1}{m} \sum_{k=1}^m |a_k| \| \bm w_k \|_1 \leqslant 2 \| f \|_{\mathcal{B}}\,,
\end{equation}
where $\bm \theta$ is sampled from the distribution related to the Barron norm.
Note that the path norm is not invariant to the parameterization.
Even for the same neural network output, the respective path norm depends on the certain parameter implementation style.
Different implementation styles lead to different representation ability, and thus the path norm is a suitable metric to describe the model capacity.}
It is natural to study two-layer neural networks with bounded $\ell_1$-path norm as the discrete version of Barron norm.
As suggested by \cite{weinan2020representation}, Barron space is the largest function space which is well/efficiently approximated by two-layer neural networks with appropriately controlled parameters via the \emph{direct} and \emph{inverse} approximation theorems \citep{weinan2021barron}.
\begin{itemize}
	\item \emph{direct approximation}: For any $f \in \mathcal{B}$, two-layer neural networks with controlled parameters are able to approximate it at a certain convergence rate without the curse of dimensionality.
	\item \emph{inverse approximation}: Any continuous function that can be efficiently approximated by two-layer neural networks with bounded $\ell_1$-path norm belongs to the Barron space.
\end{itemize}
Hence, it is natural to study two-layer ReLU neural networks with bounded $\ell_1$-path norm as the discrete version of Barron norm. In this paper, we mainly focus on the $\ell_1$-path norm.

\noindent {\bf Connection to variation spaces and smoothness:}
To control the Barron norm of two-layer ReLU networks, it is equivalent to control the total variation of the derivative for univariate functions \citep{savarese2019infinite} and the total variation norm in the Radon domain for multivariate cases \citep{ongie2020function,domingo2022tighter}.
In fact, controlling the function derivative (e.g., total variation) is common in function space theory. For example,  \cite{unser2019representer} studies the class of (univariate) functions with bounded second total variation norm; \cite{parhi2021banach} focus on the total variation norm $\|f\|_{(s)}:=c_{d}\left\|\partial_{t}^{s} \Lambda^{d-1} \mathcal{R} f\right\|_{\mathcal{M}\left(\mathbb{S}^{d-1} \times \mathbb{R}\right)}$ with $s \geqslant 2$ on on the Randon measure domain $\mathcal{M}$, where $\mathcal{R}$ is the Radon transform, $\Lambda^{d-1}$ is a ramp filter, and $\partial_{t}^{s}$ is the $s$-th partial derivative with respect to $t$, the offset variable in the
Radon domain, and $c_d$ is a dimension dependent constant.
For two-layer ReLU neural networks, this norm is equivalent to $\| f \|_{(s)} = \sum_{k=1}^m |a_k| \| \bm w_k \|_2^{s-1}$ \citep{parhi2021near}, which is the $\ell_2$-path norm by taking $s=2$.
Note that the function space in a second-order bounded variation sense is a reproducing kernel Banach space \citep{bartolucci2022understanding} and the representer theorem for data-fitting variational problem also exists \citep{parhi2021banach}.

Apart from the above function spaces in a norm-based capacity view, study on the \emph{smoothness} of function classes also exists in classical approximation theory on various spaces, e.g., H\"{o}lder, Sobolev, Besov spaces \citep{schmidt2020nonparametric,suzuki2019adaptivity}. 
Previous results have shown that the curse of dimensionality in sample complexity can be avoided for (deep) ReLU neural networks if the intrinsic dimensionality of data is small \citep{chen2019efficient,nakada2020adaptive,ghorbani2020neural} or the target function is (almost) smooth \citep{suzuki2021deep}.

\section{Empirical risk minimization under the $\ell_1$ path norm regularization}
\label{sec:probsetting}

In this paper, we consider a general random design regression problem for two-layer neural networks with the $\ell_1$-path norm in the context of statistical learning theory. Besides, we also introduce some notations needed for our analysis on metric entropy.

For regression, we employ the commonly used squared loss in this paper.
The expected risk of such hypothesis $f_{\bm \theta}$ is defined by the mean square error (MSE), i.e., $\mathcal{E}(f) = \int_Z (f_{\bm \theta}(\bm x) - y)^2 \mathrm{d} \rho $, as defined in the introduction.
The empirical risk is accordingly defined on the sample $\bm z$, i.e., $\mathcal{E}_{\bm z}(f) = \frac{1}{n} \sum_{i=1}^{n} \big(f_{\bm \theta}( \bm x_i) - y_i \big)^2$, and the excess risk is exactly the distance in $L_{\rho_X}^{2}$ under the squared loss, i.e., $\mathcal{E}(f) - \mathcal{E}(f_{\rho}) = \| f - f_{\rho} \|^2_{L^{2}_{\rho_{{X}}}}$ for any $f \in L^{2}_{\rho_{{X}}} $, where $\|f\|_{L^{2}_{\rho_{{X}}}} = \big( \int_{{X}} |f(\bm x)|^{2} \mathrm{d} \rho_{X}(\bm x) \big)^{1/2} $, where $\rho_X$ is the marginal distribution of $\rho$ on $X$, refer to \cite{cucker2007learning} for details.





The empirical risk minimization regression problem over two-layer neural networks under the $\ell_1$-path norm regularization setting is given by 
\begin{equation}\label{fzlambda}
	\bm \theta^{\star} = \argmin_{f_{\bm \theta} \in \mathcal{P}_m} \frac{1}{n} \sum_{i=1}^n (y_i - f_{\bm \theta}(\bm x_i) )^2  + \lambda \| \bm \theta \|_{\mathcal{P}}\,,
\end{equation}
where the regularization parameter $\lambda \equiv \lambda(n) > 0$ is typically assumed to satisfy $\lim_{n \rightarrow \infty} \lambda(n) = 0$.
It is clear that the solution to \cref{fzlambda} is not unique.
For example, under the ReLU setting, if $\{ (a^*_k, \bm w^*_k) \}_{k=1}^m$ corresponds to an optimal neural network, any re-scaling scheme $\{ (c_ka_k^*, \bm w^*_k/c_k) \}_{k=1}^m$ (with $c_k > 0$) is also optimal, as it does not change the neural network output and the $\ell_1$-path norm due to the positive 1-homogeneity of ReLU.

To obtain a tighter bound, we need the following \emph{projection operator} in our analysis.
\begin{definition}\cite[Projection operator]{Steinwart2008SVM}\label{proj}
	For $B \geqslant 1$, the projection operator $\pi :=  \pi_{B}$ on $\mathbb{R}$ is defined as 
	\begin{equation*}\label{BBPdef}
		\pi_B(t)= \left\{
		\begin{array}{rcl}
			\begin{split}
				& B,  ~~\text{if}~~ t > B ; \\
				& t, ~~\text{if}~~-B \leqslant t \leqslant B\\
				&-B, ~~\text{if}~~t < -B  \,.
			\end{split}
		\end{array} \right.
	\end{equation*}
	The projection of a function $f: X \rightarrow \mathbb{R}$ is given by $\pi_B(f)(\bm x) = \pi_B(f(\bm x)),~\forall \bm x \in X$.
\end{definition}
The projection operator is commonly used in learning theory, e.g., \citep{Steinwart2008SVM,shi2019sparse,liu2021gen}, beneficial to the $\| \cdot \|_{\infty}$-bounds in the convergence analysis for sharp estimation, i.e., $\| \pi_B(f_{\theta^{\star},\bm{z},\lambda}) - f_{\rho} \|^2_{L^2_{\rho_X}} $ in this work.
This is because, Eq.~\eqref{Momenthypothesis} implies $ |f_{\rho}(\bm x) = \int_{Y} y \mathrm{d} \rho(y| \bm x) | \leqslant CM := M^*$ for any $\bm x \in X$. We assume $M^* \geqslant 1$ for simplicity.
It is natural to project $f_{\bm \theta^{\star},\bm{z},\lambda}$ onto the same interval. 

Furthermore, to quantitatively understand how the complexity of $\mathcal{P}_m$ affects the learning ability of \cref{fzlambda}, we need the capacity (roughly speaking the ``size'') of $\mathcal{P}_m$ as measured by the $\ell_2$-empirical covering number \citep{koltchinskii2005complexities}. Formally, define the normalized $\ell_2$-metric $\mathscr{D}_2$ on the Euclidean space $\mathbb{R}^l$ as
\begin{equation*}
	\mathscr{D}_2(\mathbf{a}, \mathbf{b})=\left(\frac{1}{l} \sum_{i=1}^{l}\left|a_{i}-b_{i}\right|^{2}\right)^{1 / 2}, \quad \mathbf{a}=\left(a_{i}\right)_{i=1}^{l}, \mathbf{b}=\left(b_{i}\right)_{i=1}^{l} \in \mathbb{R}^{l} \,.
\end{equation*}

\begin{definition} \cite[$\ell_2$-empirical covering number]{koltchinskii2005complexities}
	For a subset $\mathcal{S}$ of a pseudo-metric space $(\mathcal{M},d)$ and $\eta > 0$, the covering number $\mathscr{N}(\mathcal{S}, \eta, d)$ is defined to be the minimal number of balls of radius $\eta$ whose union covers. For a set $\mathcal{F}$ of functions on $X$ and $\eta > 0$, the $\ell_2$-empirical covering number of $\mathcal{F}$ is given by
	\begin{equation*}
		\mathscr{N}_{2}(\mathcal{F}, \eta)=\sup _{l \in \mathbb{N}} \sup _{\bm u \in X^{l}} \mathscr{N}\left(\left.\mathcal{F}\right|_{\bm u}, \eta, \mathscr{D}_2\right)
	\end{equation*}
	where for $l \in \mathbb{N}$ and $\bm u = (u_i)_{i=1}^l \subset X^l$, we denote the covering number of the subset $\left.\mathcal{F}\right|_{\bm u}$ of the metric space $(\mathbb{R}^l, \mathscr{D}_2)$ as $ \mathscr{N}\left(\left.\mathcal{F}\right|_{\bm u}, \eta, \mathscr{D}_2\right)$.
\end{definition}

In our analysis, we also need the definition of Gaussian complexity \citep{bartlett2002rademacher}, that relates better with the metric entropy $\log \mathscr{N}_2(\mathcal{G}_{R},\epsilon)$ for $\epsilon$-covering of the index space when compared to the standard Rademacher complexity, where $\mathcal{G}_R$ is defined as $\mathcal{G}_R := \{ f_{\bm \theta} \in \mathcal{P}_m: \| \bm \theta \|_{\mathcal{P}} \leqslant R \}$.
\begin{definition} \cite[Gaussian complexity]{bartlett2002rademacher}
\label{defgaussiancom}
	The empirical Gaussian complexity of a function class $\mathcal{F}$ over data points $\{ \bm x_i \}_{i=1}^n$ is defined as
\begin{equation*}
   \mathcal{C}_n(\mathcal{F})= \mathbb{E}_{\bm \xi} \left[\sup_{f\in\mathcal{F}} \frac{1}{n} \sum_{i=1}^{n} \xi_i f(\bm x_i) \right]\,, 
\end{equation*}	
where $\bm \xi = [\xi_i, \xi_2, \cdots, \xi_n]^{\!\top}$ is a standard normal random vector.
\end{definition}
\section{Main Results}
\label{sec:mainresult}
We present our main results on sample complexity, metric entropy, and generalization bounds for learning with over-parameterized two-layer neural networks. Before presenting these results, we also need the following assumptions.

\subsection{Assumptions}
Apart from the moment hypothesis on the label noise in Eq.~\eqref{Momenthypothesis}, we consider the following two assumptions that are standard and general.

\begin{assumption} (Bounded data) \label{assbounddata} 
	 It holds that
	$\mathrm{supp}(\rho_X) \subset \{ \bm x \in \mathbb{R}^d | \| \bm x \|_{\infty} \leqslant S \}$.
\end{assumption}
For ease of convenience, we take $S=1$ throughout this paper.

\begin{assumption}\label{assrho}
	(Existence of $f_{\rho}$) We assume that the \emph{target function} $f_{\rho} \in \mathcal{B}(R) := \{f \in \mathcal{B} | \| f \|_{\mathcal{B}} \leq R\}$ exists. 
\end{assumption}
The target function class is the closure (in $L^{2}_{\rho_{{X}}}$) of any two-layer networks with a finite $\ell_1$-path norm, which includes the standard teacher-student model, e.g., \citep{tian2017analytical,loureiro2021learning,akiyama21a}.

\subsection{Sample complexity and metric entropy}
Based on the definition of Gaussian complexity, we are able to estimate the sample complexity of the path-norm based spaces (with proof deferred to \cref{app:prooflemgc}).
\begin{lemma} (sample complexity) \label{lem:gaussiancomp}
	Given the data $\{ \bm x_i \}_{i=1}^n \subseteq \mathbb{R}^d$ satisfying Assumption~\ref{assbounddata}, denote $\mathcal{G}_R = \{ f_{\bm \theta} \in \mathcal{P}_m: \| \bm \theta \|_{\mathcal{P}} \leqslant R \}$, then the empirical Gaussian complexity of $\mathcal{G}_R$ on $\{ \bm x_i \}_{i=1}^n$ is at most $\epsilon$, if the number of training data satisfies $n \geqslant \frac{8R^2\log d}{\epsilon^2}$.
\end{lemma}
{\bf Remark:} The similar result via Rademacher complexity can be given by \citep{weinan2021barron} but we still put it here as a good starting point: a width-free sample complexity bound under the $\ell_1$-path norm can be achieved.
This result enlarges the application scope of \citep{vardi2022sample} on the bounded Frobenius norm.

Similarly, our results can be directly extended to neural networks with the general $\ell_p$-path norm for $p \geq 1$, e.g., the $\ell_2$-path norm $\sum_{k=1}^m |a_k| \| \bm w_k \|_2$ used in \citep{wang2021harmless,parhi2021banach}.
We give an example of the commonly used $\ell_2$-path norm, demonstrating the sample complexity is $n \geq \frac{4R^2}{\epsilon^2}$ independent of $d$. See \cref{app:prooflemgc} for details.

Based on Lemma~\ref{lem:gaussiancomp}, the Gaussian complexity leads to a basic estimation of the metric entropy $\log \mathscr{N}_2(\mathcal{G}_1,\epsilon) \lesssim \log d \left(\frac{1}{\epsilon} \right)^2$, see the proof in \cref{app:samplecom}.
Note that this metric entropy based on the $\ell_2$-empirical covering number can be improved in the following proposition by the convex hull technique \citep{van1996weak} (with proof deferred to Appendix~\ref{app:coveringq2}).
\begin{proposition}\label{prop:coverq2} (metric entropy)
   Under Assumption~\ref{assbounddata}, denote $\mathcal{G}_R = \{ f_{\bm \theta} \in \mathcal{P}_m: \| \bm \theta \|_{\mathcal{P}} \leqslant R \}$, the metric entropy of $\mathcal{G}_1$ can be bounded by
	\begin{equation}\label{assumpN}
		 \log \mathscr{N}_2(\mathcal{G}_1,\epsilon) \leqslant C d \epsilon^{-\frac{2d}{d+2}} \,, \quad \forall \epsilon > 0 \quad \mbox{and} \quad d \geq 5\,,  
	\end{equation}
 with some universal constant $C$ independent of $d$.
\end{proposition}
{\bf Remark:} 
We make the following remarks.\\
\textit{i}) Our metric entropy provides a better estimation on $\epsilon$, and thus the derived sample complexity is better than Lemma~\ref{lem:gaussiancomp} as well as \cite{vardi2022sample}.
Besides, the estimation provides a linear dependence on the input dimension $d$, which is better than our previous arXiv version with $\mathcal{O}(d^5)$. This linear dependence on $d$ matches \cite{vardi2022sample} under practical assumptions.  \\
\textit{ii}) If we directly apply previous results on (deep) ReLU neural networks, e.g., \citep{schmidt2011convergence,suzuki2019adaptivity,bartlett2019nearly}, we have $	\log \mathscr{N}_2(\mathcal{G}_1,\epsilon) \lesssim s \log \left( \frac{md}{\epsilon} \right)$, where $s$ is the number of \emph{free} parameters in our two-layer neural network or  $\| \bm a \|_0 + \| \bm W \|_0 \leqslant s$. 
The sparsity $s = o(m)$ holds true for deep ReLU neural networks as they have few activations \citep{hanin2019deep} but normally this is not valid for our over-parameterized two-layer setting. In this case, we only have $s = \Omega(m)$, leading to a vacuous convergence rate $\mathcal{O}(m/n)$.

\subsection{Convergence rates of the excess risk}
\label{sec:excessrisk}

Based on the above results, now we can state our main results on over-parameterized two-layer neural networks in the Barron space.
We follow \cite{schmidt2020nonparametric,chen2019efficient,suzuki2021deep} on generalization guarantees for neural networks that directly assume the attainable property of the global minimum.
In the next subsection, we will develop a computational algorithm to obtain a global minimum of the original non-convex optimization problem. 

Our result on the generalization guarantees under the ReLU activation function is given as below (with proofs deferred to \cref{app:maintheorem}).
\begin{theorem}\label{maintheo}
	Considering problem~(\ref{fzlambda}) in $\mathcal{G}_R = \{ f_{\bm \theta} \in \mathcal{P}_m: \| \bm \theta \|_{\mathcal{P}} \leqslant R \}$ for over-parameterized two-layer ReLU neural networks with $d \geq 5$, under Assumptions~\ref{assbounddata},~\ref{assrho}, and moment hypothesis in \cref{Momenthypothesis}, taking $R \geqslant CM \geqslant 1$, a large $B$ and any $0 < \delta < 1$, with probability $1 - \delta$, there exists one global minimum $\bm \theta^{\star}$ of problem~(\ref{fzlambda}) such that
	\begin{equation*}
		\big\| \pi_{B} ( f_{\bm \theta^{\star}} ) - f_{\rho} \big\|_{L^2_{\rho_X}}^2 \lesssim \lambda + \frac{R^2}{m} + R^2 d^{\frac{1}{3}} n^{-\frac{d+2}{2d+2}}  \log \frac{4}{\delta} +  BCM^2 \exp \Big(\! -\frac{B}{4CM^2} \!\Big) \,.
	\end{equation*}
\end{theorem}
{\bf Remark:} We make the following remarks:\\
\textit{i}) The first two term of RHS involves the regularization error which relates to the approximation ability.
The third term of RHS is the estimation of the sample error, depending on generalization error; and the last term of RHS involves the output error.
For example, taking $\lambda :=1/n$ (faster than $\mathcal{O}(n^{-\frac{d+2}{2d+2}})$ is enough), we can obtain a certain convergence rate at $\mathcal{O}(n^{-\frac{d+2}{2d+2}})$ of the excess risk under our over-parameterized setting.
Note that our result depends on $d$: 1) the convergence rate tends to $\mathcal{O}(1/\sqrt{n})$ when $d$ is large enough, and thus is immune to CoD.
2) the dependence is $d^{1/(2+q)}$ with $q:=\frac{2d}{d+2}$ but we write $d^{1/3}$ here just for simplicity.
Instead, kernel methods estimators can not evade CoD due to the lower bound $\Omega(n^{-\frac{1}{d}})$. Accordingly, the separation between kernel methods and neural networks is built in the perspective of function space for approximation.
This separation can be also studied via sample complexity \citep{yehudai2019power}, Kolmogorov width \citep{wu2022spectral}.\\
\textit{ii}) The pre-given radius $R$ in $\mathcal{G}_R$ can be properly chosen by the standard iteration technique \citep{Wu2006Learning} in approximation theory, leading to a certain rate at $\mathcal{O}(n^{\epsilon - \frac{d+2}{2d+2}})$ for some sufficiently small $\epsilon$. We do not include this result in our paper as it is classical and standard.

\noindent {\bf Discussion on the improved approximation rate:}
There are some literature building the improved approximation rate as well as metric entropy.
To be specific, given some constants $c_1, c_2 \in \mathbb{R}$, considering the following function space in \citep{siegel2021sharp}
\begin{equation*}
    \mathbb{T} := \overline{ \left\{ \sum_{i=1}^m a_i \sigma(\bm w_i^{\!\top} \bm x + b_i), \bm w_i \in \mathbb{S}^{d-1}, b \in [c_1, c_2], \sum_{i=1}^m |a_i| \leq 1 \right\} } \,,
\end{equation*}

which is the closure of the convex, symmetric hull of two-layer ReLU neural networks with the bounded $\ell_1$ path norm. The estimation of the metric entropy is given by \cite{siegel2021sharp}
\begin{equation}\label{eq:optme}
   \epsilon^{-\frac{2d+3}{2d}} {_d \! \lesssim} \log \mathscr{N}_2(\mathbb{T},\epsilon) \lesssim_d \epsilon^{-\frac{2d+3}{2d}}\,,
\end{equation}
where $\lesssim_d$ denotes that some constant depending on $d$ is omitted.
Though both our result and \cite{siegel2021sharp} use the smoothness structure of the symmetric convex hull, the techniques are different.
\cite{siegel2021sharp} employ the orthogonal argument for ridge functions in Hilbert space and then the metric entropy can be lower bounded by the minimum Hilbert norm of these nearly orthogonal ridge functions.
Instead, our results exploit the relationship between the symmetric convex hull and VC-hull class, and then the metric entropy of the convex hull of any polynomial class is of lower order than $\epsilon^{-r}$ with $r < 2$.
It appears that an improved estimation is provided by \cite{siegel2021sharp}, however, the dependence on $d$ is unclear and difficult to be checked. 
If the derived result depends on $c^d$ for some $c > 1$, the metric entropy as well as the generalization analysis still suffers from CoD. The claim for approximation and generalization cannot be well supported in this case.
Besides, the dependence on $d$ can be clearly calculated at most the polynomial order \citep{wu2022spectral} via a different spectral decomposition approach based on spherical harmonics.
Nevertheless, this result is applied to Kolmogorov width but is still unanswered for metric entropy.

Based on this optimal metric entropy in \cref{eq:optme} as well as the optimal approximation in \citep{siegel2021sharp}, it can be applied to the variation norm based space equipped with two-layer ReLU neural networks, e.g., \citep[Theorem 8]{parhi2021near}, \citep[Theorem 4.2]{yang2024optimal}.
This leads to the minimax rate at $\mathcal{O}(n^{-\frac{d+3}{2d+3}})$ for the excess risk. However, the dependence on $d$ is still unclear due to the use of \citep{siegel2021sharp}.
To our knowledge, our result provides the best estimation on metric entropy of Barron space on the trade-off between $\epsilon$ and $d$.
Apart from this main difference, our results also differ from them in the label noise type as well as a computational algorithm that will be described as below.

\subsection{A computational algorithm}
	As mentioned before, we directly assume that one global optimal solution can be obtained.
	In this subsection, we devise a certain algorithm based on the measure representation \citep{pilanci2020neural,akiyama21a,zweig2021functional}. Note that, the developed algorithm is not the main contribution of this work, but bridges the gap between theoretical and practical use of learning in the Barron space.
	
Normally, optimization in the Barron space is difficult, or even NP-hard, e.g., the conditional gradient algorithm (a.k.a. Frank-Wolfe algorithm) developed by \cite{bach2017breaking}. 
	If the standard gradient descent (or gradient flow) is employed, the CoD can not be avoided \citep[Theorem 1]{wojtowytsch2020can} or an exponentially large number of widths in terms of $n$ and $d$ is required \citep{akiyama21a,takakura2024mean} under mean field analysis.
    Besides, \cite{akiyama2022excess} develop a two-phase noisy gradient descent algorithm, which provably reaches the near-optimal solution. Nevertheless, the considered function space in their work is smaller than the Barron space as kernel methods do not suffer from CoD in that space; besides the convergence rate $\mathcal{O}(m^5/n)$ of the excess risk in their work is not applicable under our \emph{over-parameterized} setting.
	
	Here we directly employ one computational algorithm developed by \citep{pilanci2020neural} that uses the measure representation of over-parameterized, two-layer ReLU neural networks and convex duality. This type algorithm based on the strong duality works in a high dimensional convex program but actually admits approximations that are successful in practice \citep{mishkin2022fast}.
To be specific, let $\bm X = [\bm x_1, \bm x_2, \cdots, \bm x_n]^{\!\top} \in \mathbb{R}^{n \times d}$ be the data matrix, and $\{ \bm D_i \}_{i=1}^P$ be the set of diagonal matrices whose diagonal is given by  $[ \mathrm{1}_{\{\bm x_1^{\top} \bm w \geqslant 0\}}, \dots, \mathrm{1}_{\{\bm x_n^{\top} \bm w \geqslant 0\}}]$ for all possible $\bm w \in \mathbb{R}^d$ through the scaling-invariance of ReLU. There is a finite number $P$ of such matrices. Under the basic over-parameterization condition with $m \geqslant n+1$, strong duality holds \citep{rosset2007l} and then a globally optimal solution for the non-convex optimization problem \eqref{fzlambda} can be solved by a high dimensional convex program (see \cref{app:opt} for details)
	\begin{equation}\label{fzfinitecom1}
		\begin{split}
			\{ \bm u_i^* \}_{i=1}^{2P} \!=\! \argmin_{\bm u_i \in \mathcal{C}_i }\, { \frac{1}{n}}\Big \| \sum_{i=1}^{2P} \bm D_i\bm X \bm u_i \!-\! \bm y \Big\|_2^2 \!+\! \lambda \sum_{i=1}^{2P}  \|\bm u_i\|_1 \,,
		\end{split}
	\end{equation}
where $\mathcal{C}_i = \{ \bm u \in \mathbb{R}^d \large| (2\bm D_i- \bm I_n)\bm X \bm u \geqslant 0   \}$, $\mathcal{C}_{i+p} = \mathcal{C}_i$, $\forall i \in [P]$ by setting $\bm D_{i+p} = -\bm D_i$.
We need to remark that this convex program has $2dP$ variables and $2nP$ linear inequalities, where $P = 2r[e(n-1)/r]^r$ and $r = \mathrm{rank}(\bm X)$.
If the data matrix is low rank, this problem can be solved efficiently.
Accordingly, we are ready to present our generalization results for two-layer ReLU neural networks via a convex program (with proofs deferred to \cref{app:maintheorem}).
\begin{proposition}\label{mainprop} [optimization]
	Given an over-parameterized ($m \geqslant n+1$), two-layer, ReLU neural network with $d>5$ in \cref{fzlambda} endowed by $\mathcal{G}_R = \{ f_{\bm \theta} \in \mathcal{P}_m: \| \bm \theta \|_{\mathcal{P}} \leqslant R \}$ with a global minima $f_{\bm \theta}^{(T)}(\bm x) = \frac{1}{m} \sum_{k=1}^m a_k^{(T)} \sigma( \langle \bm w_k^{(T)}, \bm x \rangle ) $ by solving a high dimensional convex problem~(\ref{fzfinitecom1}) with $\mathcal{O}(d r(n/r)^r)$ variables and $\mathcal{O}(n r(n/r)^r)$ linear inequalities with $r=\mbox{rank}(\bm X)$ after $T$ iterations. Under Assumptions~\ref{assbounddata},~\ref{assrho}, and moment hypothesis in \cref{Momenthypothesis}, and taking $R \geqslant CM \geqslant 1$, then for a large $B$ and any $0 < \delta < 1$, the following result holds with probability $1 - \delta$
	\begin{equation*}
		\big\| \pi_{B} ( f_{\bm \theta^{(T)}} ) - f_{\rho} \big\|_{L^2_{\rho_X}}^2 \lesssim \lambda + \frac{R^2}{m} + CMR d^{\frac{1}{3}} n^{-\frac{d+2}{2d+2}}  \log \frac{4}{\delta} +  BCM^2 \exp \Big(\! -\frac{B}{4CM^2} \!\Big) + \mathcal{O} \left( \frac{1}{T^2} \right) \,.
	\end{equation*}
\end{proposition}
{\bf Remark:}
We make the following remarks:\\
\textit{i)} Our results are still valid if stochastic approximation algorithms are used for solving problem~(\ref{fzfinitecom1}), see \cref{app:opterror} for details. 
Though the convergence rate $\mathcal{O}(1/T^2)$ with $T$ iterations can be achieved for optimization, the computational complexity would be high as we mentioned before. Optimization over the Barron space is difficult and such convex program admits $\mathcal{O}(d r(n/r)^r)$ variables and $\mathcal{O}(n r(n/r)^r)$ linear inequalities with $r=\mbox{rank}(\bm X)$. Such problem can be efficiently solved if the data are low-rank. \\
\textit{iii)} Note that, different from the mean field analysis, the training dynamics at each iteration cannot be well posed under the convex program, which could be a drawback. 
Nevertheless, all of globally optimal solutions could be obtained by a similar convex program by permutation and splitting/merging of the neurons \citep{wang2020hidden}, which provides a possible direction to study the scaling law of certain datasets.
We hope it would facilitate the future research on the optimization over the Barron space in terms of statistical-computational gap.\\
\textit{iii)} If we consider the classical $\ell_2$ regularization (i.e., weight decay) in the original non-convex objective function, the related high dimensional optimization problem only differs in its regularizer, i.e., $\lambda \sum_{i=1}^{2P} \| \bm u_i \|_2$.
Such differences in fact relate to the classical $\ell_1$ \emph{vs}. $\ell_2$ regularization in terms of sample complexity \citep{gopi2013one,wei2019regularization}.

\section{Proof framework}
\label{sec:proofframe}
In this section, we establish the framework of proofs for Theorem~\ref{maintheo}.
Since ~\cref{maintheo} is the special case of Proposition~\ref{mainprop} without employing a certain optimization algorithm, we consider the proofs of Proposition~\ref{mainprop} but leave the optimization error to \cref{app:opterror}.

\underline{\bf 1) Error decomposition:} The excess risk in the Barron space can be upper bounded by four terms via the following decomposition scheme (with proofs deferred to \cref{app:errdecom}).
\begin{proposition}\label{properrdec}
	The excess risk $  \mathcal{E} [ \pi_{B} ( f_{\bm \theta^{(T)}} ) ] - \mathcal{E}(f_{\rho})$ can be decomposed as
	\begin{equation*}
		\begin{split}
			\mathcal{E} [ \pi_{B} ( f_{\bm \theta^{(T)}} ) ]  - \mathcal{E}(f_{\rho}) 
			& \leqslant {\tt Opt}(\bm z, \lambda) + {\tt  Out}  (y) + {\tt D}(\lambda) + {\tt S}(\bm z, \lambda, \bm \theta)  \,,
		\end{split}
	\end{equation*}
	where ${\tt Opt}(\bm z, \lambda) :=    \mathcal{E} [ \pi_{B} ( f_{\bm \theta^{(T)}} ) ]    - \mathcal{E} [\pi_{B} ( f_{\bm{\theta}^{\star}} ) ]$ is the optimization error for a global optimal solution $f_{\bm{\theta}^{\star}}$; ${\tt Out}(y) := \frac{1}{n} \sum_{i=1}^n | \pi_{B}(y_i) - y_i|^2$ is the output error; $	{\tt D}(\lambda):=\inf_{f \in \mathcal{P}_m} \Big\{ \mathcal{E}(f) - \mathcal{E}(f_{\rho}) + \lambda \| \bm \theta \|_{\mathcal{P}} \Big\} $ is the regularization error; and 
	the sample error is ${\tt S}(\bm z, \lambda, \bm \theta) := \mathcal{E}\big[ \pi_{B} (f_{\bm{\theta}^{\star}} ) \big] -  \mathcal{E}_{\bm{z}}\big[ \pi_{B} (f_{\bm{\theta}^{\star}} ) \big]  + \mathcal{E}_{\bm{z}}\big(f^{\lambda}_{\bm \theta}\big) - \mathcal{E}\big(f^{\lambda}_{\bm \theta}\big) $ with $f^{\lambda}_{\bm \theta} = \argmin_{f_{\bm \theta} \in \mathcal{P}_m} \Big\{ \mathcal{E}(f_{\bm \theta}) - \mathcal{E}(f_{\rho}) + \lambda \| \bm \theta \|_{\mathcal{P}} \Big\} $.
\end{proposition}

In the following, we provide the key ideas needed to estimate these error terms.

\underline{\bf 2) Output error:} We focus on the output error ${\tt Out}(y) := \frac{1}{n} \sum_{i=1}^n \left| \pi_{{B}}(y_i) - y_i \right|^2$, which is more intractable due to the squared order. In this case, the random variable $(|y_i| - B)^2 \mathbb{I}_{\{ |y_i| \geqslant B \} } $ is no longer sub-exponential but still admits the exponential-type tail decay. 
We introduce sub-Weibull random variables \citep{vladimirova2020sub,zhang2020concentration} to tackle this issue, with the proof deferred to~\cref{app:outputerr}.
\begin{proposition}[Output error]\label{propoutput} 
	Let $B \geqslant CM \geqslant 1$. Under moment hypothesis in~\cref{Momenthypothesis}, there exists a subset $Z_1$ of $Z^n$ with probability at least $1 - {\delta}/{4}$ such that
	\begin{equation*}
		\begin{aligned}
			\frac{1}{n} \sum_{i=1}^n \left| \pi_{{B}}(y_i) - y_i \right|^2 \! \lesssim \!  \frac{1}{n} \log^{2} \frac{4}{\delta} \!+\! \exp \left( -\frac{B}{CM^2} \right) (CM^2)^{2}  \,, ~~\forall \bm z := (\bm x, y) \in Z_1\,.
		\end{aligned}
	\end{equation*}
\end{proposition}
{\bf Remark:} The derived error bound is $\mathcal{O}(1/n) + \mathrm{Err}$, where the residual term $\mathrm{Err}$ admits an exponential decay w.r.p to $B$, which is better than previous results on unbounded outputs in \citep{Wang2011Optimal,guo2013concentration,liu2021gen} with $\mathrm{Err} := 2^k B^{-k} k^k M^k$. They require $B:= n^{\epsilon}$ (increasing slowly) and large $k$ such that $B^{-k}$ would behave like $1/n$. 
However, this needs $n \geqslant (2kM)^k$, and hence $\mathrm{Err}$ cannot easily tend to zero in prior results as the sample complexity suffers from CoD when $k \geqslant d$.

\underline{\bf 3) Regularization error:} This term can be estimated by the approximation properties in the Barron space: denote $\bm \theta_{\rho}$ as the parameter of $f_{\rho}$, we have ${\tt D}(\lambda) \leqslant \lambda \| f_{\rho} \|_{\mathcal{P}} + \frac{3\| \bm \theta_{\rho} \|^2_{\mathcal{P}}}{m}$ due to $f_{\rho} \in \mathcal{B}(R)$ in Assumption~\ref{assrho} and the approximation error in \cite[Theorem 1]{weinan2021barron}.

\underline{\bf 4) Sample error:} Estimation of the sample error is also one key part in our proof (see Appendix~\ref{app:sampleerr}).
The techniques differ from previous learning theory literature in terms of the function space, the estimation for truncated outputs, and the complexities of the target function.
\begin{proposition}\label{propsampleerr}
	Under Assumptions~\ref{assbounddata},~\ref{assrho} and moment hypothesis in~\cref{Momenthypothesis}, let $R \geqslant B \geqslant M \geqslant 1$ and $CM \geqslant 1$. Then, there exists a subset of $Z' $ of $Z^{n}$ with confidence at least $1-3\delta/4$ with $0 < \delta < 1$ such that for any $\bm z:=(\bm x, y) \in Z'$ and $f_{\bm{\theta}^{\star}} \in \mathcal{G}_R$ 
\begin{equation*}
	{\tt S}(\bm z, \lambda, \bm \theta) \lesssim (CM)^2 \left( \frac{S}{n} \log \frac{4}{\delta} + \lambda \right) + Rd n^{-\frac{d+2}{2d+2}}  \log \frac{4}{\delta} +  BCM^2 \exp \Big(\! -\frac{B}{4CM^2} \!\Big)\,.
\end{equation*}
\end{proposition}
Combining the above results, we conclude the proof of Theorem~\ref{maintheo} (as well as Proposition~\ref{mainprop}).

\section{Numerical Validation}
\label{sec:exp}

In this section, we conduct numerical experiments to validate our theoretical results in the perspective of the convergence rate of the excess risk.

To validate whether the derived (sharper) convergence rate is attainable or not, we construct a simple synthetic dataset under a known $f_{\rho}$ in the over-parameterized regime.
To be specific, we assume that the data are sampled from a normal Gaussian distribution, i.e., $\bm x \sim \mathcal{N}(\bm 0, \bm I_d)$ and normalized with $\| \bm x \|_2 = 1$.
The feature dimension is $d=3$, a low dimension setting to ensure $P$ in \cref{theooptsol} is not large as mentioned before.
We set the number of training points to range from $10$ to $1000$ while the number of test points is held fixed at $20$.
Albeit simple, such an experimental setting still works in the over-parameterized regime, see Table~\ref{table:reslut}~{\color{blue}(Left)}.
We consider the noiseless case, where the target function is generated by a single ReLU, i.e., $y = f_{\rho}(\bm x) = \sigma(\langle \bm w^*, \bm x \rangle )$ with $\bm w^* \sim \mathcal{N}(\bm 0,\bm I_d)$.
The regularization parameter is set to $\lambda= 10^{-8}$ for both two methods, kernel ridge regression via the NTK and the path norm based algorithm.
We solve the convex program in ~\cref{fzfinitecom} using CVX~\citep{grant2014cvx} to obtain the exact global minima and then compute the test MSE for regression over $5$ runs.

\begin{table}
	\vspace{-2mm}
	\begin{minipage}{.2\textwidth}
		\scalebox{0.7}{\label{table:param}
			\centering
			\begin{tabular}{ccccccc}
				\hline
				$n$ (\#training data) & 10 & 20  & 30  & 40 & 50 \\
				\hline
				$m$ (\#parameters) & 32 & 116 & 192 & 250 & 325\\
				\hline
			\end{tabular}
		}
	\end{minipage}\hfill \hspace{-255mm} \vspace{-3mm}
	\begin{minipage}{.30\textwidth}
		\includegraphics[width=0.93\textwidth]{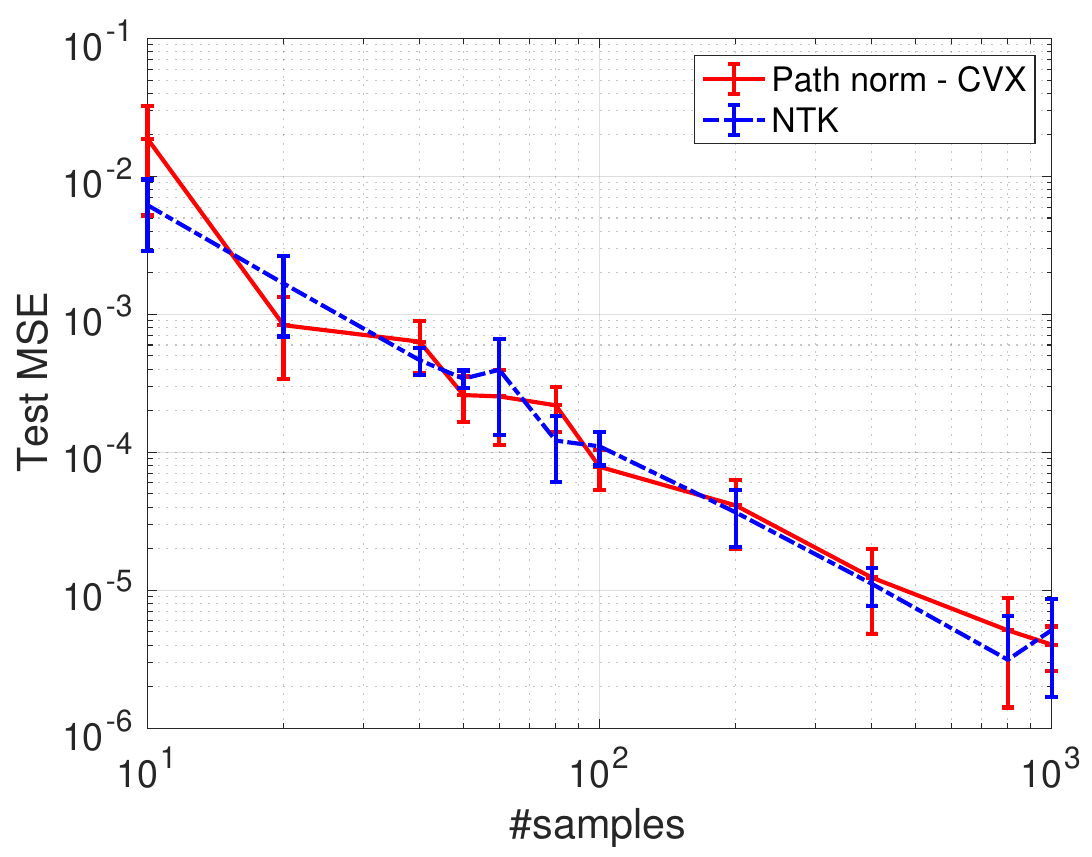}
		\label{fig:rate}
	\end{minipage}
	\begin{minipage}{.30\textwidth}
	\includegraphics[width=1.09\textwidth]{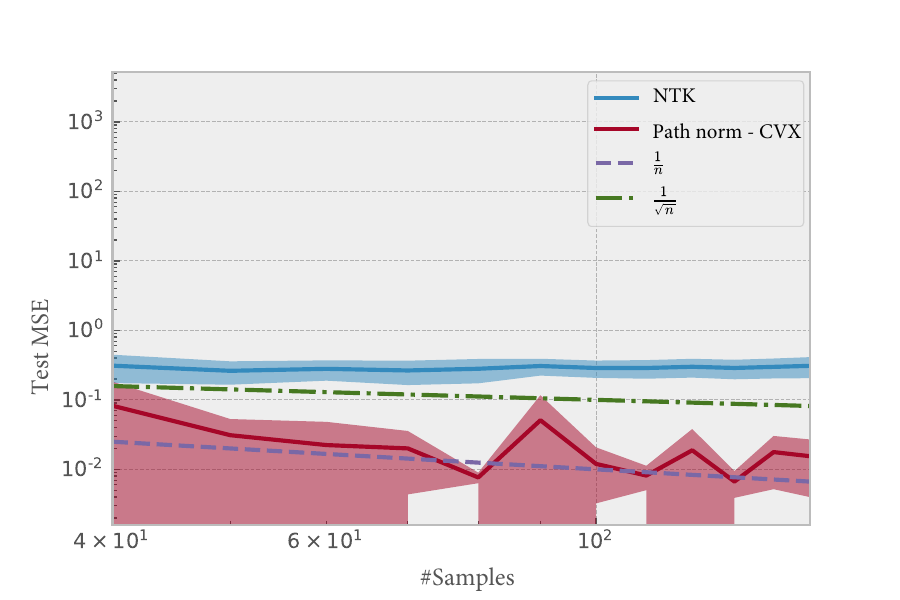}
\label{fig:opt}
\end{minipage}
	\caption{{\color{blue} Left:} the number of (activated) parameters \emph{v.s.} the number of training data in the synthetic dataset. Convergence rates of problem~\eqref{fzlambda} under the path norm \emph{v.s.} NTK on a synthetic dataset ({\color{blue} Middle}) and the UCI ML Breast Cancer dataset ({\color{blue} Right}), respectively.}
	\label{table:reslut}
\end{table}

The {\color{blue}~(middle)} figure of Table~\ref{table:reslut} shows that, when learning a single ReLU beyond RKHS, our algorithm still achieves the same convergence rate as the NTK in RKHS regime. This is because, the input dimension $d=3$ is not large, so there is no significant difference on the convergence rate. 

Besides, we also conduct this experiment on a real-world dataset, i.e., the UCI ML Breast Cancer dataset with 569 samples and the dimension $d=30$.
We set 80\% of samples used for training and 20\% of samples for test.
Here the number of training data ranges from 40 to 300, and the number of test data ranges from 10 to 75, accordingly.
The remaining experimental setting is the same as that of the synthetic dataset.

The {\color{blue}~(right)} figure of Table~\ref{table:reslut} shows that, when increasing the number of training data, the test MSE of NTK slightly decreases.
Instead, the path norm based algorithm achieves a significant lower test MSE, which demonstrates the attainability of our theoretical results.
Nevertheless, we also need to point out that, the path norm based algorithm is quite inefficient and unstable when compared to NTK. 
The performance is based on an extreme accurate solution by CVX, which restricts the utility of this convex program algorithm in practice.
Additionally, we remark here that we do not claim this algorithm is better than SGD.

\section{Conclusion and discussion}
\label{sec:conclusion}
This work provides a theoretical understanding on the separation between kernel methods and neural networks from the perspective of function space.
Our work sheds light on the theoretical guarantees of learning with over-parameterized two-layer neural networks under a general norm based capacity, and demonstrates the possibility of achieving sharper convergence rates with a clear dependence on $d$ via a computational high-dimensional convex algorithm.

Our results have several findings, 1) learning with the $\ell_1$-path norm is able to achieve faster convergence rate than $\mathcal{O}(\frac{1}{\sqrt{n}})$ on the excess risk under the general setting.
2) while kernel methods suffers from CoD, including the popular random features model, the NTK approach and other kernel estimators, neural networks can avoid this CoD and outperform kernel estimators from the perspective of function space theory.
Hence, we hope that our analysis opens the door to improved analysis of neural networks in a norm-based capacity view and function spaces in the machine learning community. 

We admit that transforming from the parameter space to the measure space leads to a very high dimensional convex problem, but the developed computational algorithm nonetheless provides a possible way to obtain the global minima.
Optimization over the Barron space is quite difficult. Maybe the Barron space is still a bit large from the perspective of optimization.
Identifying a suitable function space that is \emph{data-adaptive} than RKHS as well as computationally efficient to learn a high-dimensional function, e.g., \citep{spek2022duality,chen2023duality} has its own interest in approximation theory and deep learning theory and still is unanswered well, see more references \citep{steinwart2024reproducing,scholpple2023spaces}.
This is always our target to understand approximation-optimization trade-off and statistical-computational gap from kernel methods to neural networks, e.g., from the perspective of single/multi-index models \citep{abbe2022merged,lee2024neural,damian2024computational,bietti2023learning}.

\section*{Acknowledgement}
We thank the anonymous reviewers for their constructive feedback and Denny Wu for his engagement on the dimension dependence.
This work was supported by the Hasler Foundation Program: Hasler Responsible AI (project number 21043), by the Army Research Office and was accomplished under Grant Number W911NF-24-1-0048, by the Swiss National Science Foundation (SNSF) under grant number 200021\_205011.

\appendix


\section{Estimation of sample complexity and covering number}
\label{app:cover}

\subsection{Proof of Lemma~\ref{lem:gaussiancomp}}
\label{app:prooflemgc}
Here we provide the upper bound of the Gaussian complexity and then transform this bound to sample complexity.

\begin{proof}[Proof of Lemma~\ref{lem:gaussiancomp}]
	According to \cite[Lemma 1]{barron2019complexity}, due to 1-Lipschitz of the ReLU activation function $\sigma$, we have
	\begin{equation}\label{lem:contraction}
		\mathcal{C}_n(\sigma \circ \mathcal{G}_R) \leqslant \mathcal{C}_n(\mathcal{G}_R)\,.
	\end{equation}
	Based on the definition of $\mathcal{C}_n(\mathcal{G}_R)$ in Definition~\ref{defgaussiancom}, we have
	\begin{equation*}
		\mathcal{C}_n(\mathcal{G}_R)= \mathbb{E}_{\bm \xi} \left[\sup_{f_{\bm \theta}\in\mathcal{G}_R} \frac{1}{n} \sum_{i=1}^{n} \xi_i f_{\bm \theta}(\bm x_i) \right] = \mathbb{E}_{\bm \xi} \left[\sup_{f_{\bm \theta}\in \mathcal{G}_R} \frac{1}{m} \sum_{k=1}^{m} a_k \frac{1}{n} \sum_{i=1}^{n} \xi_i \sigma(\bm w_k^\top \bm x_i) \right] \,,
	\end{equation*}
	which can be further upper bounded by
	\begin{equation*}
		\begin{split}
			\mathcal{C}_n(\mathcal{G}_R) & \leqslant \mathbb{E}_{\bm \xi} \left[\sup_{f_{\bm \theta}\in \mathcal{G}_R} \frac{1}{m} \sum_{k=1}^{m} |a_k|\|\bm w_k\|_1 \frac{1}{n} \left| \sum_{i=1}^{n} \xi_i \sigma(\frac{\bm w_k}{\|\bm w_k\|_1}^\top \bm x_i) \right|\right] \\
			& \leqslant R\; \mathbb{E}_{\bm \xi} \left[\sup_{\|\bm w\|_1 \leqslant 1}\frac{1}{n} \left| \sum_{i=1}^{n} \xi_i \sigma({\bm w}^\top \bm x_i) \right|\right] \\
			& \leqslant 2R\; \mathbb{E}_{\bm \xi} \left[\sup_{\|\bm w\|_1 \leqslant 1}\frac{1}{n} \sum_{i=1}^{n} \xi_i \sigma({\bm w}^\top \bm x_i)\right] \; \text{[using symmetry of Gaussians]} \\
			&\leqslant 2R\; \mathbb{E}_{\bm \xi} \left[\sup_{\| \bm w\|_1 \leqslant 1}\frac{1}{n} \sum_{i=1}^{n} \xi_i {\bm w}^\top \bm x_i\right]\; \text{[using \cref{lem:contraction}]} \\
			&\leqslant 2R\; \mathbb{E}_{\bm \xi} \left[ \left\|\frac{1}{n} \sum_{i=1}^{n} \xi_i \bm x_i \right \|_\infty \right] \mbox{[using H\"{o}lder's inequality]}\\
			& = 2R\; \mathbb{E}_{\bm \xi} \left[ \max_{k=1, \dots d} \frac{1}{n} \sum_{i=1}^{n} \xi_i x^{(k)}_i \right] \\
			& \leqslant 2R \sqrt{\frac{2 \log d}{n}} \,, \quad \mbox{[using Assumption~\ref{assbounddata}]}
		\end{split}
	\end{equation*}
	where the last inequality uses the upper bounds for $d$ sub-Gaussian maxima, see \cite[Exercise 2.12]{wainwright2019high}.
	Taking the RHS of the above equation as $\epsilon$, we conclude the proof.
\end{proof}
{\bf Remark:} If we consider other types of the path norm for the two-layer ReLU neural networks, e.g., $\ell_p$-path norm, we can obtain the similar result by
\begin{equation*}
    \mathbb{E}_{\bm \xi} \sup_{\| \bm w\|_p \leq 1} \left\langle \bm w, \frac{1}{n} \sum_{i=1}^n \xi_i \bm x_i \right\rangle \leq \mathbb{E}_{\bm \xi} \| \frac{1}{n} \sum_{i=1}^n \xi_i \bm x_i \|_q = \frac{1}{n} \mathbb{E}_{\bm \xi} \| \bm X^{\!\top} \bm \xi\|_q\,,
\end{equation*}
where we use H\"{o}lder's inequality.
Then we can estimate $\mathbb{E}_{\bm \xi} \| \bm X^{\!\top} \bm \xi\|_q$ under different $q$ to conclude the proof.
For the commonly used $\ell_2$-path norm, we have $\frac{1}{n} \mathbb{E}_{\bm \xi} \| \bm X^{\!\top} \bm \xi\|_2 \leq \frac{1}{n}  \sqrt{\mbox{tr}(\bm X^{\!\top} \bm X)} \leq \frac{1}{\sqrt{n}} $ due to \cref{assbounddata}. 
Accordingly, we have $\mathcal{C}_n(\sigma \circ \mathcal{G}_R) \leq \frac{2R}{\sqrt{n}}$ under the $\ell_2$-path norm. Finally we conclude the proof.

\subsection{Relation between Gaussian complexity and metric entropy}
\label{app:samplecom}

Here we build the connection between Gaussian complexity and metric entropy by the following proposition, which would be beneficial to our proof.

\begin{proposition}\label{propcomplexity}
	Under Assumption~\ref{assbounddata},  denote $\mathcal{G}_R = \{ f_{\bm \theta} \in \mathcal{P}_m: \| \bm \theta \|_{\mathcal{P}} \leqslant R \}$, then the metric entropy of $\mathcal{G}_1$ is estimated by
	\begin{equation*}
		\log \mathscr{N}_2(\mathcal{G}_1,\epsilon) \leqslant c \log d \left( \frac{1}{\epsilon} \right)^2, \quad \forall \epsilon>0\,, \quad \mbox{with $c:= 32\pi \log 2$}\,.    
	\end{equation*}
\end{proposition}

To prove \cref{propcomplexity}, we need lower bound the empirical Gaussian complexity using Gaussian comparison theorems. To this end, we need the following metric on packing number \cite[Definition 5.4]{wainwright2019high}.

For $\epsilon > 0$, denote $\mathscr{D}_2(\mathcal{G}_R, \epsilon)$ as the cardinality of the largest $\epsilon$-packing of $\mathcal{G}_R$, it admits 
\begin{equation*}
	\mathscr{N}_2(\mathcal{G}_R, \epsilon) \leqslant \mathscr{D}(\mathcal{G}_R,\epsilon)\,.
\end{equation*}

\begin{lemma}\label{lem:packing}
	Let $ \{ \bm x_i \}_{i=1}^n \subseteq \mathbb{R}^d$ satisfying Assumption~\ref{assbounddata}, denote $\mathcal{G}_R = \{ f_{\bm \theta} \in \mathcal{P}_m: \| \bm \theta \|_{\mathcal{P}} \leqslant R \}$, then for any $\epsilon > 0$, we have 
	\[
	\mathcal{C}_n(\mathcal{G}_1) \geqslant  \frac{1}{2\sqrt{\pi \log 2}} \frac{\epsilon}{\sqrt{n}} \sqrt{\log \mathscr{N}_2(\mathcal{G}_1, \epsilon)} \,.
	\]
\end{lemma}

\begin{proof}
	Let $\epsilon > 0$. Let $\mathscr{D}$ be the largest $\epsilon$-packing of $\mathcal{G}_1$. Due to $\mathscr{D} \subseteq \mathcal{G}_1$, we have that
	\[
	\mathcal{C}_n(\mathcal{G}_1) \geqslant \mathbb{E} \left[\sup_{f_{\bm \theta}\in\mathscr{D}} \frac{1}{n} \sum_{i=1}^{n} Z_i f_{\bm \theta}(\bm x_i) \right]\,.
	\]
	Denote $X_f:= \frac{1}{n} \sum_{i=1}^{n} Z_i f_{\bm \theta}(\bm x_i)$, we can observe that for any $f, g \in \mathscr{D}$,
	\[
	\mathbb{E}[(X_f - X_g)^2] \geqslant \frac{1}{n} \epsilon^2.
	\]
	Define iid random variables $(Y_f)_{f\in\mathscr{D}}$ such that 
	\[
	Y_f = \frac{\epsilon}{2\sqrt{n}} W_f \,, \quad \text{with}~W_f \sim \mathcal{N}(0, 1)\,,
	\]
 then we have
	\[
	\mathbb{E} \left[\sup_{f\in\mathscr{D}} X_f \right] \geqslant \mathbb{E} \left[\sup_{f\in\mathscr{D}} Y_f \right]\,.
	\]
	Since the condition for Sudakov-Fernique's comparison in Lemma~\ref{lem:sudakov} is verified, we have the following result by Lemma~\ref{lem:gaussians}
	\[
	\mathbb{E}\left[\sup_{f\in\mathscr{D}} Y_f \right] = \frac{\epsilon}{2\sqrt{n}} \mathbb{E} \left[\sup_{f \in \mathscr{D}} Z_f \right] \geqslant \frac{1}{2\sqrt{\pi \log 2}} \frac{\epsilon}{\sqrt{n}} \sqrt{\log \mathscr{D}(\mathcal{G}_1, \epsilon)}\,,
	\]
	which concludes the proof by recalling that $\mathscr{D}(\mathcal{G}_1, \epsilon) \geqslant \mathscr{N}_2(\mathcal{G}_1, \epsilon)$.
\end{proof}

Combing Lemmas~\ref{lem:gaussiancomp} and \ref{lem:packing}, we can directly finish the proof of Proposition~\ref{propcomplexity}.

\subsection{Proof of Proposition~\ref{prop:coverq2}}
\label{app:coveringq2}

To improve the estimation of the metric entropy in Proposition~\ref{propcomplexity}, we need a useful lemma \citep{van1996weak} on the convex hull technique. 
The considered convex hull is sequentially closed and its envelope has a weak second moment, and thus this function class is Donsker. Accordingly, the metric entropy of such convex hull of any polynomial class is of lower order than $\epsilon^{-r}$ with $r<2$, which is enough to ensure that Dudley’s entropy integral is finite.

\begin{lemma}  \citep[Theorem 2.6.9]{van1996weak} \label{lem:convexcovering}
Let $\rho_X$ be a probability measure over the input space $\bm x \in X$, and $\mathcal{F}$ be a class of measurable functions with a measurable squared integrable envelope $F$ such that $\rho_X F^2 < \infty$ and for some $V>0$,
\begin{equation*}
    \mathscr{N}(\mathcal{F}, \epsilon \| F \|_{L^2_{\rho_X}}, L^2_{\rho_X})  \leqslant C \left( \frac{1}{\epsilon} \right)^V\,, \quad 0 < \epsilon < 1\,.
\end{equation*}
Then there exists a constant $C_{V}$ depending on $C$ and $V$ such that
\begin{equation}\label{eq:vv2}
    \log \mathscr{N} (\overline{\mathrm{conv}}(\mathcal{F}), \epsilon \| F \|_{L^2_{\rho_X}}, L^2_{\rho_X}) \leqslant C_{V} \left( \frac{1}{\epsilon} \right)^{\frac{2V}{V+2}}\,.
\end{equation}
Here we rewrite it by clearly indicating the constant.
Let $W := 1/2 + 1/V$ and $L := C^{1/V} \| F \|_{L^2_{\rho_X}}$, the function space $\mathcal{F}$ can be covered by $n$ balls of radius at most $L n^{-1/V}$.
Form the set $\mathcal{F}_1 \subset \mathcal{F}_2 \subset \cdots \subset \mathcal{F}$ such that $\mathcal{F}_n$ is a maximal, $L n^{-1/V}$-separated net over $\mathcal{F}$. Then \cref{eq:vv2} is equivalent to prove
\begin{equation*}
     \log \mathscr{N} (C_k L n^{-W}, \overline{\mathrm{conv}}(\mathcal{F}_{nk^q}), L^2_{\rho_X}) \leqslant D_k n \,, \quad n,k \geq 1\,,
\end{equation*}
where $\{C_k\}_{k=1}^{\infty}$ and $\{D_k\}_{k=1}^{\infty}$ are two sequences admitting
\begin{equation*}
    \begin{split}
        C_k & = C_{k-1} + \frac{1}{k^2}\,, \\
        D_k & = D_{k-1} + 2^{2q/V+1} \frac{1+ \log (1+k^{2q/V-4+q})}{k^{2q/C-4}} \,,
    \end{split}
\end{equation*}
where $q$ is some constant satisfying $q \geq 3+V$ and we need to choose proper parameters to ensure the convergence of $\{C_k\}_{k=1}^{\infty}$ and $\{D_k\}_{k=1}^{\infty}$.
\end{lemma}

Now we are ready to prove Proposition~\ref{prop:coverq2}.
\begin{proof}[Proof of Proposition~\ref{prop:coverq2}]
Recall the considered function space $\mathcal{P}_{m}$
\begin{equation}\label{eqfmnew}
	\mathcal{P}_{m} = \left\{ f_{\bm \theta}(\bm \cdot) = \frac{1}{m}\sum_{k=1}^m a_k \sigma \big( \langle \bm w_k, \bm \cdot \rangle \big) = \frac{1}{m}\sum_{k=1}^m \widetilde{a}_k \sigma \left( \left \langle \widetilde{\bm w}_k, \bm \cdot \right \rangle \right) \right\} \,,
\end{equation}
where $\widetilde{a}_k = \| {\bm w}_k \|_1 a_k /m $ and $\widetilde{\bm w}_k = {\bm w}_k/{\| {\bm w}_k \|_1}$ by the scaling variance property of ReLU for any $\bm w_k \in \mathbb{R}^d/\{\bm 0\}$, $\forall k \in [m]$, so we have
$\| \bm \theta \|_{\mathcal{P}} = \sum_{k=1}^m |\widetilde{a}_k| $.
In this case, we consider $\mathcal{G}_R = \{ f \in \mathcal{P}_m: \| \bm \theta \|_{\mathcal{P}} \leqslant R \} $, and the related parameter space $\widetilde{\bm w} \in \mathcal{W} = \mathbb{S}^{d-1}_1$ is the $\ell_1$-norm ball. \footnote{If $\| {\bm w}_k \|_1 = 0$ with $k \in [m]$, we can directly set $\widetilde{a}_k = 0$ and $\widetilde{\bm w}_k = \bm 0$. In this case, we only consider non-zero $\widetilde{\bm w}_k$ for convenience.}

We consider the following function space 
\begin{equation*}
\mathcal{F} = \{ \sigma(\langle \widetilde{\bm w}, \cdot \rangle): \widetilde{\bm w} \in \mathcal{W} \} \cup \{ 0 \} \cup \{ - \sigma(\langle \widetilde{\bm w}, \cdot \rangle) : \widetilde{\bm w} \in \mathcal{W} \} \,,
\end{equation*}
and the convex hull of $\mathcal{F}$ is 
\begin{equation}\label{eq:l1regu}
    \overline{\mathrm{conv}}\mathcal{F} = \left\{ \sum_{i=1}^m \alpha_i f_i \bigg| f_i \in \mathcal{F}, \sum_{i=1}^{m} \alpha_i =1, \alpha_i \geqslant 0, m \in \mathbb{N} \right\}\,. 
\end{equation}
It can be found that $\mathcal{G}_1 \subset \overline{\mathrm{conv}} \mathcal{F}$, and hence we focus on bounding the covering number of $\overline{\mathrm{conv}} \mathcal{F}$.
One can see that the convex hull of $\mathcal{F}$ in Eq.~\eqref{eq:l1regu} is actually the same as the data-dependent hypothesis space with the $\ell_1$ coefficient regularization \citep{shi2011concentration} and we employ the result here.

Recall that in our setting the data are bounded in Assumption~\ref{assbounddata}, which implies that the used ReLU in our work is 1-Lipschitz continuous with respect to the weights. Consequently, if $\{ \widetilde{\bm w}_j \}_{j=1}^m$ is an $\epsilon$-net of $\mathcal{W}$, then $\{ \sigma( \langle \widetilde{\bm w}_j, \cdot \rangle) \}_{j=1}^m$ is an $\epsilon$-net of $ \{ \sigma( \langle \widetilde{\bm w}, \cdot \rangle) : \widetilde{\bm w} \in \mathcal{W} \}$ in ${\tt C}(X)$.
Accordingly, $\mathcal{F}_m$, defined as
\begin{equation*}
\mathcal{F}_m = \{ \sigma(\langle \widetilde{\bm w}_j, \cdot \rangle) \}_{j=1}^m \cup \{ 0 \} \cup \{ - \sigma(\langle \widetilde{\bm w}_j, \cdot \rangle) \}_{j=1}^m \,,
\end{equation*}
is an $\epsilon$-net of $\mathcal{F}$ in ${\tt C}(X)$. That means,
\begin{equation*}
    \mathscr{N} (\mathcal{F}, \epsilon, \| \cdot \|_{\infty}) \leqslant 2 \mathscr{N} (\mathcal{W}, \epsilon) + 1\,,
\end{equation*}
where $\mathscr{N} (\mathcal{W}, \epsilon)$ is the covering number of $\mathcal{W}$ with respect to the Euclidean distance.
Due to $\mathcal{W} = \mathbb{S}^{d-1}_1$ in our problem, we have  $\mathscr{N} (\mathbb{S}^{d-1}_1, \epsilon) \leqslant (3/\epsilon)^d$ \cite[Chapter 5]{wainwright2019high}.

The class $\mathcal{F}$ satisfies the conditions of Lemma~\ref{lem:convexcovering} with $F$ being the constant $1$ as both weights and data are bounded.
Accordingly, we have
\begin{equation*}
   \mathscr{N}(\mathcal{F}, \epsilon \| F \|_{L^2_{\rho_X}}, L^2_{\rho_X}) \leqslant 2^{d+1} \left( \frac{1}{\epsilon} \right)^d + 1\,.
\end{equation*}
By taking $V:=d$ and $C := 2^{d+1} + 1$ in Lemma~\ref{lem:convexcovering} and the empirical measure $\mu:= \frac{1}{n} \sum_{i=1}^n \delta_{\bm x_i}$, then we have
\begin{equation}\label{eqlogn2}
    	\log \mathscr{N}_2(\mathcal{G}_1,\epsilon) \leqslant \log \mathscr{N}_2(\overline{\mathrm{conv}}\mathcal{F},\epsilon, \mu) \leqslant c \left( \frac{1}{\epsilon} \right)^{\frac{2d}{d+2}} \,,
\end{equation}
where $c$ is some constant depending on $d$.
In the next, we will track this dimension dependence.

To ensure our estimate in \cref{eqlogn2} non-vacuous, we need to carefully check the proof of \cite[Theorem 2.6.9]{van1996weak} to ensure that $c$ cannot depend on the exponential order of $d$. To be specific, we have
\begin{equation*}
    \log \mathscr{N}_2(\overline{\mathrm{conv}}\mathcal{F},\epsilon, \mu) \leqslant D_k [C_k (2^{d+1}+1)^{\frac{1}{d}}]^{\frac{2d}{d+2}} \left( \frac{1}{\epsilon} \right)^{\frac{2d}{d+2}}\,,
\end{equation*}
where $\{C_k\}_{k=1}^{\infty}$ and $\{D_k\}_{k=1}^{\infty}$ are two converged sequences under some proper parameters admitting
\begin{equation*}
    \begin{split}
        C_k & = C_{k-1} + \frac{1}{k^2}\,, \\
        D_k & = D_{k-1} + 64 \frac{1+ \log (1+k^{2+3d})}{k^2} \,.
    \end{split}
\end{equation*}
By some calculation, we have $C_k \leq C_1 + 2$ for any $k$, and $D_k$ admits
\begin{equation*}
D_k \leq D_{k-1} + \frac{64}{k^2} + \frac{192(1+d)\log k}{k^2} \leq D_{k-1} + \frac{64}{k^2} + \frac{384(1+d)}{k^{3/2}} \leq D_1 + 896 + 768d\,.
\end{equation*}
Therefore, we only need to check $C_1$ and $D_1$, respectively.

For $C_1$, the proof of \cite[Theorem 2.6.9]{van1996weak} requires that the following inequality holds
\begin{equation}\label{eq:cvcons}
    e^{n/d} (6d^{\frac{1}{2} + \frac{1}{d}}) \left( e + \frac{eC_1^2}{d^{\frac{2}{d}}} \right)^{8 d^{\frac{2}{d}} C_1^{-2}n} \leq e^{n}\,,
\end{equation}
where $\mathcal{F}$ is covered by $n$ balls with the radius at most $(2^{d+1}+1)^{\frac{1}{d}} n^{-\frac{1}{d}}$.
Then \cref{eq:cvcons} is equivalent to
\begin{equation*}
    8 d^{\frac{2}{d}} C_1^{-2} \log \left(e + \frac{eC_1^2}{d^{\frac{2}{d}}} \right) \leq 1 - \frac{1}{d} \left[ 1+ \log 6 + ( \frac{1}{2} + \frac{1}{d}) \log d \right]\,.
\end{equation*}
Clearly, taking $C_1$ in a constant order is sufficient to ensure this inequality for any $d \geq 5$. For example, we can choose $C_1 = 50$.
For $D_1$, \cref{eq:cvcons} implies that taking $D_1 = 1$ is enough, precisely, using Eq.~(2.6.10) in \cite[Theorem 2.6.9]{van1996weak}.

Based on the above discussion, we have for any $d \geq 5$
\begin{equation*}
\begin{split}
     \log \mathscr{N}_2(\overline{\mathrm{conv}}\mathcal{F},\epsilon, \mu) 
     & \leq D_k [C_k (2^{d+1}+1)^{\frac{1}{d}}]^{\frac{2d}{d+2}} \left( \frac{1}{\epsilon} \right)^{\frac{2d}{d+2}} \\
     & \leq (897+768d)[52 (2^{d+1}+1)^{\frac{1}{d}}]^{\frac{2d}{d+2}} \left( \frac{1}{\epsilon} \right)^{\frac{2d}{d+2}} \\
     & \leq 10^7 d \left( \frac{1}{\epsilon} \right)^{\frac{2d}{d+2}} \,.
\end{split}
\end{equation*}
Finally we conclude the proof. 
\end{proof}

\section{Optimization via a high dimensional convex program}
\label{app:opt}

In this section, we give the exact equivalence between the original non-convex optimization problem and a high dimensional convex problem by adapting the work of \citep{pilanci2020neural} to our setting. 

 Let $\bm X = [\bm x_1, \bm x_2, \cdots, \bm x_n]^{\!\top} \in \mathbb{R}^{n \times d}$ be the data matrix, and $\{ \bm D_i \}_{i=1}^P$ be the set of diagonal matrices whose diagonal is given by  $[ \mathrm{1}_{\{\bm x_1^{\top} \bm w \geqslant 0\}}, \dots, \mathrm{1}_{\{\bm x_n^{\top} \bm w \geqslant 0\}}]$ for all possible $\bm w \in \mathbb{R}^d$. There is a finite number $P$ of such matrices.  We then have the following proposition on the globally optimal solutions for the non-convex optimization problem \eqref{fzlambda}.
\begin{proposition}\label{theooptsol}
	In the over-parameterized regime with $m \geqslant n+1$, the non-convex regularized problem \eqref{fzlambda} admits a global minimum that is recovered from the solution of the convex program
	\begin{equation}\label{fzfinitecom}
		\begin{split}
			\{ \bm u_i^{\star} \}_{i=1}^{2P} \!=\! \argmin_{\bm u_i \in \mathcal{C}_i }\, { \frac{1}{n}}\Big \| \sum_{i=1}^{2P} \bm D_i\bm X \bm u_i \!-\! \bm y \Big\|_2^2 \!+\! \lambda \sum_{i=1}^{2P}  \|\bm u_i\|_1 \,,
		\end{split}
	\end{equation}
where $\mathcal{C}_i = \{ \bm u \in \mathbb{R}^d \large| (2\bm D_i- \bm I_n)\bm X \bm u \geqslant 0   \}$, $\mathcal{C}_{i+p} = \mathcal{C}_i$, $\forall i \in [P]$ by setting $\bm D_{i+p} = -\bm D_i$.
	 The solution $\{ ( a^{\star}_k, \bm w^{\star}_k ) \}_{k=1}^m$ can be constructed by non-zero elements of $\{ \bm u^*_i \}_{i=1}^{2P}$ in~\cref{fzfinitecom} such that  $m:= \sum_{i: \bm u^{\star}_i \neq \bm 0}^{2P} 1$. To be specific, denote the index set $\mathcal{J} = [j_1, j_2, \cdots, j_m]$ for non-zero elements of $\{ \bm u^{\star}_i \}_{i=1}^{2P}$ with $j_1 \geqslant 1$ and $j_m \leqslant 2P$, we have for any $ k \in [m]$, the weights are $ \bm w^{\star}_{k}  = \bm u^{\star}_{j_k}/\| \bm u^{\star}_{j_k} \|_1$; and $a^{\star}_k = m\| \bm u^{\star}_{j_k} \|_1$ if $j_k \leqslant P$ and $a^{\star}_k = -m\| \bm u^{\star}_i \|_1$ if $j_k \geqslant P+1$.
\end{proposition}
{\bf Remark:}
\textit{i}) The number of neurons $m$ is not fixed but larger than $\sum_{i: \bm u^{\star}_i \neq 0}^{2P} 1$ (also larger than $n+1$) to satisfy the strong duality. When transformed to a convex program in measure spaces \citep{pilanci2020neural,chizat2021convergence}, the problem has $2dP$ variables and $2nP$ linear inequalities, where $P = 2r[e(n-1)/r]^r$ and $r = \mathrm{rank}(\bm X)$. If the data are low-dimensional or low-rank, then $P$ is not very large. 
We remark that, compared to most previous work going beyond RKHSs assuming that a global minimum is given, we provide this computational approach for analysis. \\
\textit{ii}) Solving problem~\cref{fzfinitecom} can find an optimal neural network $f_{\bm \theta^{\star}} (\bm x)= \frac{1}{m} \sum_{k=1}^m a_k^{\star} \sigma( \langle \bm w_k^{\star}, \bm x \rangle ) $.
In fact, all globally optimal neural networks (i.e., all of globally optimal solutions) could be obtained by a similar convex program by permutation and splitting/merging of the neurons \citep{wang2020hidden}.
Our analysis just focuses on one global minimum obtained by solving \cref{fzfinitecom} for description simplicity.\\

In the next, we first give some explanations on the additional $\ell_2$ regularization in~\cref{fzfinitecom} in Section~\ref{app:optl2e} and then present proofs of \cref{theooptsol} for the regularization problem in Section~\ref{app:optinter}.

\subsection{Additional $\ell_2$ regularization in~\cref{fzfinitecom}}
\label{app:optl2e}

The optimal solutions of~\cref{fzfinitecom} may not be unique, which would increase the difficulty for the generalization analysis.
To overcome this issue, we add an extra $\ell_2$ regularization term, i.e., \emph{elastic net penalty} \citep{zou2005regularization}, into the objective function problem~(\ref{fzfinitecom}) to make it strongly convex.
This scheme is common in optimization \citep{bruer2014time} and learning theory \citep{de2009elastic,rosasco2019convergence}.

\cref{fzfinitecom} can be equivalently transformed to the following composite convex program
\begin{equation}\label{fzfinitecomthree}
	\begin{split}
		\min_{ \{\bm u_i \}_{i=1}^{2P} }   { \frac{1}{n}}\Big \| \sum_{i=1}^{2P} \bm D_i\bm X \bm u_i - \bm y \Big\|_2^2 + \lambda \sum_{i=1}^{2P}  \|\bm u_i\|_1 + \sum_{i=1}^{2P} \iota_{\mathcal{C}_i}(\bm u_i) \,,
	\end{split}
\end{equation}
where the first term is convex, smooth, gradient Lipschitz continuous; the second term is convex but nonsmooth; and the third term is related to the indicator function $\iota_{\mathcal{C}_i}: \mathbb{R}^d \rightarrow \mathbb{R} \cup \{+\infty \}$ defined as
\begin{equation*}
	\iota_{\mathcal{C}_i} (\bm u_i): = \left\{
	\begin{array}{rcl}
		\begin{split}
			0 \quad &\text{if}~~ \bm u_i \in \mathcal{C}_i ; \\
			+\infty \quad  &\text{otherwise}\,,
		\end{split}
	\end{array} \right.
\end{equation*}
where $\mathcal{C}_i = \{ \bm u \in \mathbb{R}^d \large| (2\bm D_i- \bm I_n)\bm X \bm u \geqslant 0   \}$, $\mathcal{C}_{i+p} = \mathcal{C}_i$, $\forall i \in [P]$ by setting $\bm D_{i+p} = -\bm D_i$.

Note that, since \cref{fzfinitecomthree} is not strongly convex as $\bm D_i \bm X$ might be not full rank, it could be difficult to obtain convergence on sequence for optimization error estimation, which increases the difficulty on generalization analysis.
Thankfully, the strongly convexity property can be obtained considering an \emph{elastic net penalty} \citep{zou2005regularization}, that is adding a small strongly convex term to the sparsity inducing penalty. This will lead to a nice sequence convergence in optimization. Thereby, \cref{fzfinitecomthree} can be transformed to
\begin{equation}\label{fzfinitecomthreelam}
	\begin{split}
		\min_{ \{\bm u_i \}_{i=1}^{2P} }  \underbrace{ { \frac{1}{n}}\Big \| \sum_{i=1}^{2P} \bm D_i\bm X \bm u_i - \bm y \Big\|_2^2 + \widetilde{\lambda} \sum_{i=1}^{2P}  \|\bm u_i\|_2^2}_{\triangleq g( \{ \bm u_i\}_{i=1}^{2P} )} + \lambda \sum_{i=1}^{2P}  \|\bm u_i\|_1 + \sum_{i=1}^{2P}\iota_{\mathcal{C}_i}(\bm u_i) \,,
	\end{split}
\end{equation}
where $\widetilde{\lambda}$ is an additional regularization parameter, which can be sufficiently small, and the function $g$ is hence strongly convex.
Interestingly, in some cases, e.g., the exact recovery problem \citep{bruer2014time}, adding \emph{elastic net penalty} does not change the optimal solution.
In fact, this penalty is not only used in optimization, but also common in learning theory as the strongly convex property on the empirical risk does not always hold. This is because, such requirement depends on the probability measure $\rho$ and is typically not satisfied in high (possibly infinite) dimensional settings, see \citep{de2009elastic,rosasco2019convergence} for details.
Accordingly, we do not strictly distinguish the difference between \cref{fzfinitecom} and~\cref{fzfinitecomthreelam} in this paper.

\subsection{Proof of Proposition~\ref{theooptsol} }
\label{app:optinter}
In this subsection, we use the measure representation of two-layer neural networks via convex duality \citep{pilanci2020neural,akiyama21a}.
The proof technique follows \citep{pilanci2020neural} except the $\ell_1$-path norm regularization.
For self-completeness, we provide a brief proof here.
	\begin{proof}
		We re-write \cref{fzlambda} as 
		\begin{equation}\label{fzpath}
			p^{\star} = \min_{\{a_k\}_{k=1}^m, \{ \bm w_k \}_{k=1}^m } \frac{1}{n} \sum_{i=1}^n \left(y_i - \sum_{k=1}^m a_k \sigma(\bm w_k^{\!\top} \bm x_i) \right)^2 + \frac{\lambda}{m} \sum_{k=1}^m |a_k| \| \bm w_k \|_1\,,
		\end{equation}
		which is equivalent to the following optimization problem
		\begin{equation}\label{fzpathdual1}
			p^{\star} =  \min_{\|\bm w_k\|_1 \leqslant 1} \max_{ \substack{\bm v \in \mathbb{R}^n \,\mbox{\scriptsize s.t.} \\ \vert \bm v^{\!\top}(\bm X \bm w_k)_+\vert \leqslant \lambda,~\forall k \in [m]  }} -\frac{n}{4} \left\| \bm v - \frac{2\bm y}{n} \right\|_2^2 + \frac{1}{n} \left\|\bm y \right\|_2^2\,.
		\end{equation}
		Interchanging the order of min and max in \cref{fzpathdual1}, we obtain the lower bound $d^{\star}$ via weak duality
		\begin{equation}\label{fzpathweakdual1}
			p^{\star} \geqslant d^{\star} =  \max_{ \substack{\bm v \in \mathbb{R}^n \,\mbox{\scriptsize s.t.} \\ \vert \bm v^{\!\top}(\bm X \bm w)_+\vert \leqslant \lambda, \forall \bm w,  \| \bm w \|_1 \leqslant 1  }} -\frac{n}{4} \left\| \bm v - \frac{2\bm y}{n} \right\|_2^2 + \frac{1}{n} \left\|\bm y \right\|_2^2\,.
		\end{equation}
		If $m \geqslant n+1$, the strong duality holds \citep{rosset2007l}, i.e., $d^{\star} = p^{\star}$.
		
		Similar derivation from \citep{pilanci2020neural}, we have
		\allowdisplaybreaks
		\begin{align*}
			&\min_{\substack{\bm \zeta, \bm \zeta^\prime \in \mathbb{R}^P\\ \bm \zeta, \bm \zeta^\prime \geqslant 0}} \, \max_{\substack{\bm v\in \mathbb{R}^n \\ \bm \alpha_i, \bm \beta_i \in {\mathbb{R}^n}\\ \bm \alpha_i, \bm \beta_i \geqslant 0, \, \forall i \in [M]\\ \bm \alpha_i^\prime, \bm \beta_i^\prime \in {\mathbb{R}^n}\\ \bm \alpha_i^\prime, \bm \beta_i^\prime \geqslant 0, \, \forall i \in [P] }} \, \min_{\substack{\bm r_i \in \mathbb{R}^d,\, \|\bm r_i\|_1 \leqslant 1\\ \bm r_i^\prime \in \mathbb{R}^d,\, \|\bm r_i^\prime\|_1 \leqslant 1\\ \forall i \in [P]} } -\frac{n}{4} \left\| \bm v - \frac{2\bm y}{n} \right\|_2^2 + \frac{1}{n} \left\|\bm y \right\|_2^2 
			\\
			& \hspace{5.3cm}+ \sum_{i=1}^P \zeta_i \big(\lambda + \bm r_i^{\!\top} \bm X^{\!\top} \bm D(\mathcal{S}_i) \big( \bm v + \bm \alpha_i + \bm \beta_i \big)- \bm r_i^{\!\top} \bm X^{\!\top} \bm \beta_i  \big)\\
			&\hspace{5.3cm}+ \sum_{i=1}^P \zeta^{\prime}_i \big(\lambda + \bm r_i^{{\prime}\!\top} \bm X^{\!\top} \bm D(\mathcal{S}_i) \big( -\bm v + \bm \alpha'_i + \bm \beta'_i \big)- \bm r_i^{{\prime}\!\top} \bm X^{\!\top} \bm \beta'_i  \big) \,,
		\end{align*}
		where we use the dual norm of $\ell_{\infty}$: $\| \bm x \|_{\infty} = \sup_{\bm z}  \{ |\bm z^{\!\top} \bm x| \big| \| \bm x \|_1 \leqslant 1 \}$.
		The constraints $\| \bm r_i \|_1 \leqslant 1$ and $\| \bm r'_i \|_1 \leqslant 1$ are convex and compact, and optimization over $\bm v, \bm \alpha_i, \bm \beta_i$, $\bm r_i, \bm r'_i, \forall i \in [M]$ are convex.
		Accordingly, we can exchange the order of the inner $\max \min$ problem as a $\min \max$ problem.
		By maximizing over $\bm v, \bm \alpha_i, \bm \beta_i, \bm \alpha'_i, \bm \beta'_i$,  the dual optimization problem in~\cref{fzpathweakdual1} can be formulated as
		
		\begin{align*}
			&\min_{\substack{\bm \zeta, \bm \zeta^\prime \in \mathbb{R}^P\\ \bm \zeta, \bm \zeta^\prime \geqslant 0}} \,\min_{\substack{\bm r_i \in \mathbb{R}^d,\, \|\bm r_i\|_1 \leqslant 1\\ \bm r_i^\prime \in \mathbb{R}^d,\, \|\bm r_i^\prime\|_1 \leqslant 1 \\  (2 \bm D(\mathcal{S}_i)- \bm I_n) \bm X \bm r_i \geqslant 0\\  (2 \bm D( \mathcal{S}_i)- \bm I_n)\bm X \bm r_i^\prime \geqslant 0}}
			{\frac{1}{n}}\Big\|\sum_{i=1}^P \zeta_i \bm D(\mathcal{S}_i) {\bm X} \bm r_i -\zeta_i^\prime \bm D(\mathcal{S}_i) \bm X \bm r_i^{\prime} - \bm y \Big\|_2^2 + \lambda \sum_{i=1}^P (\zeta_i+\zeta_i^\prime) \,,
		\end{align*}
			where we use the fact that, the constraint $(2 \bm D(\mathcal{S}_i)- \bm I_n) \bm X \bm r \leqslant 0$ is equivalent to $( \bm D(\mathcal{S}_i)- \bm I_n) \bm X \bm r \leqslant 0$ and $\bm D(\mathcal{S}_i) \bm X \bm r \leqslant 0$ for any $\bm r$ since $ \bm D(\mathcal{S}_i)- \bm I_n$ and $\bm D(\mathcal{S}_i)$ have no overlap. Then we have the constraint $(2 \bm D(\mathcal{S}_i)- \bm I_n) \bm X \bm r \geqslant 0$ by flipping the sign of $\bm r$.
		Likewise, we have
		\begin{equation*}
			\min_{\substack{\bm u_i, \bm u_i^\prime \in \mathcal{P}_{\mathcal{S}_i} }}
			\frac{1}{n}\Big\|\sum_{i=1}^P \bm D(\mathcal{S}_i) \bm X (\bm u_i^\prime- \bm u_i) - \bm y \Big\|_2^2 
			+ \lambda \sum_{i=1}^P (\| \bm u_i \|_1 + \| \bm u'_i \|_1)\,,
		\end{equation*}
		which concludes the proof by taking the compact form.
	\end{proof}
	
\section{Proofs of~\cref{maintheo}}
\label{app:maintheorem}
In this section, we give the proofs for \cref{maintheo}.
In fact, this theorem is a special case of Proposition~\ref{mainprop}, and accordingly we present the proofs here for
the error decomposition in Section~\ref{app:errdecom}, the error bounds for the optimization error in Section~\ref{app:opterror}, the output error in Section~\ref{app:outputerr}, and the sample error in Section~\ref{app:sampleerr}, respectively.

Before presenting our error bound, we give a formal definition of regularization error for notational simplicity.

\begin{definition}\label{assreg} (regularization error)
	The regularization error of $\mathcal{P}_m$ is defined as
	\begin{equation}\label{Dlamdadef}
		{\tt D}(\lambda):=\inf_{f_{\bm \theta} \in \mathcal{P}_m} \Big\{ \mathcal{E}(f_{\bm \theta}) - \mathcal{E}(f_{\rho}) + \lambda \| \bm \theta \|_{\mathcal{P}} \Big\} \,.
	\end{equation}
	For any $\lambda > 0$, the \emph{regularizing function} is defined as
	\begin{equation}\label{flamdadef}
		f_{\bm \theta}^{\lambda} = \argmin_{f_{\bm \theta} \in \mathcal{P}_m} \Big\{ \mathcal{E}(f_{\bm \theta}) - \mathcal{E}(f_{\rho}) + \lambda \| \bm \theta \|_{\mathcal{P}}  \Big\}\,.
	\end{equation}
\end{definition}
The decay of ${\tt D}(\lambda)$ as $\lambda \rightarrow 0$ measures the approximation ability of the function space $\mathcal{P}_m$.
We have a natural result for the regularization error from the approximation ability of the Barron space
\begin{equation}\label{Dlambda}
	{\tt D}(\lambda)  = \inf_{f_{\bm \theta} \in \mathcal{P}_m} \Big\{ \| f_{\bm \theta} - f_{\rho} \|^2_{L^2_{\rho_X}} + \lambda \| f_{\bm \theta} \|_{\mathcal{P}}  \Big\} \leqslant \lambda \| f_{\rho} \|_{\mathcal{P}} + \frac{3 \| \bm \theta_{\rho} \|^2_{\mathcal{P}}}{m}  ,~~\forall \lambda>0\,,
\end{equation}
where $\bm \theta_{\rho}$ denotes the parameter of the target function $f_{\rho}$.
This is because $f_{\rho} \in \mathcal{B}(R)$ in Assumption~\ref{assrho} and the basic approximation performance of the Barron space \cite[Theorem 1]{weinan2021barron}. 
In the literature of learning theory for regularized kernel based methods \citep{cucker2007learning,Steinwart2008SVM}, the regularization error is assumed to be satisfied ${\tt D}(\lambda):=\inf_{f \in \mathcal{H}} \Big\{ \mathcal{E}(f) - \mathcal{E}(f_{\rho}) + \lambda \| f \|^2_{\mathcal{H}} \Big\}  \leqslant \mathcal{O}(\lambda^{\beta})$ with $0 < \beta \leqslant 1$, where $\mathcal{H}$ is a RKHS associated with a positive definite kernel and $\| \cdot \|_{\mathcal{H}}$ is a Hilbert norm.
In our work, this assumption naturally holds because of the off-the-shelf approximation result in the Barron space.

\subsection{Proof of the error decomposition}
\label{app:errdecom}

Here we give the proof of the error decomposition presented in Proposition~\ref{properrdec}.
Before presenting the results, we decompose the sample error ${\tt S}(\bm z, \lambda, \bm \theta)$ into the following two parts ${\tt S}(\bm z, \lambda, \bm \theta) =  {\tt S}_1(\bm z, \lambda, \bm \theta)  +  {\tt S}_2(\bm z, \lambda, \bm \theta) $ with
\begin{equation*}
	\begin{split}
		&  {\tt S}_1(\bm z, \lambda, \bm \theta) \!=\!  \mathcal{E}\big[ \pi_{B} (f_{\bm \theta^{\star}} ) \big] \!-\!  \mathcal{E}(f_{\rho}) -  \mathcal{E}_{\bm{z}}\big[ \pi_{B} (f_{\bm \theta^{\star}} ) \big] +    \mathcal{E}_{\bm{z}}(f_{\rho}) \,,\\
		&  {\tt S}_2(\bm z, \lambda, \bm \theta) = \Big\{ \mathcal{E}_{\bm{z}}\big(f^{\lambda}_{\bm \theta}\big) - \mathcal{E}_{\bm{z}}(f_{\rho}) \Big\} - \Big\{ \mathcal{E}(f^{\lambda}_{\bm \theta}) - \mathcal{E}(f_{\rho}) \Big\}\,.
	\end{split}
\end{equation*}
Now we are ready to present our proof of Proposition~\ref{properrdec}.
\begin{proof}[Proof of Proposition~\ref{properrdec}]
	According to the project operator $\pi_B$ in Definition~\ref{proj}, we have $0 \leqslant \pi_B(a)-\pi_B(b) \leqslant a-b$ if
	$a \geqslant b$; and $a-b \leqslant \pi_B(a)-\pi_B(b) \leqslant 0$ if $a \leqslant b$. That means, for the squared loss, there holds
	\begin{equation*}
		\left( \pi_{B}(f_{\bm \theta}(\bm x)) - \pi_{B}(y)  \right)^2 \leqslant (f_{\bm \theta}(\bm x) - y)^2\,,
	\end{equation*}
	which implies
	\begin{equation}\label{outputy}
		\begin{split}
			\mathcal{E}_{\bm z}[\pi_{B}(f_{\bm \theta})] & = \frac{1}{n} \sum_{i=1}^n  \left( \pi_{B}(f_{\bm \theta}(\bm x)) - y  \right)^2 \leqslant \frac{1}{n} \sum_{i=1}^n  \left( \pi_{B}(f_{\bm \theta}(\bm x)) - \pi_{B}(y)   \right)^2 + \frac{1}{n} \sum_{i=1}^n \left| \pi_{{B}}(y_i) - y_i \right|^2 \\
			& \leqslant 	\mathcal{E}_{\bm z}(f_{\bm \theta}) + \frac{1}{n} \sum_{i=1}^n \left| \pi_{{B}}(y_i) - y_i \right|^2 \,,
		\end{split}
	\end{equation}
	where the second term is termed as the output error.
	Accordingly, we write the excess risk as
	\begin{equation*}
		\begin{split}
			\mathcal{E} [ \pi_{B} ( f_{\bm{\theta}^{(T)}} ) ]   - \mathcal{E}(f_{\rho}) & \leqslant   \mathcal{E} [ \pi_{B} ( f_{\bm{\theta}^{(T)}} ) ]  - \mathcal{E}(f_{\rho})  + \lambda \| f_{\bm \theta^{\star}} \|_{\mathcal{P}} \\
			& = {\tt Opt}(\bm z, \lambda) +  \mathcal{E}  [ \pi_{B} ( f_{\bm{\theta}^{\star}} ) ] - \mathcal{E}(f_{\rho})  + \lambda \| f_{\bm \theta^{\star}} \|_{\mathcal{P}} \\
			& ={\tt Opt}(\bm z, \lambda) + \mathcal{E}  [ \pi_{B} ( f_{\bm{\theta}^{\star}} ) ]   - \mathcal{E}_{\bm z}  [ \pi_{B} ( f_{\bm{\theta}^{\star}} ) ]  - \mathcal{E}(f_{\rho})  +  \mathcal{E}_{\bm z}  [ \pi_{B} ( f_{\bm{\theta}^{\star}} ) ] + \lambda \| f_{\bm \theta^{\star}} \|_{\mathcal{P}} \\ 
			& \leqslant {\tt Opt}(\bm z, \lambda) + \mathcal{E}  [ \pi_{B} ( f_{\bm{\theta}^{\star}} ) ]   - \mathcal{E}_{\bm z}  [ \pi_{B} ( f_{\bm{\theta}^{\star}} ) ]  - \mathcal{E}(f_{\rho})  +  \mathcal{E}_{\bm z}  ( f_{\bm{\theta}^{\star}} ) + \lambda \| f_{\bm \theta^{\star}} \|_{\mathcal{P}} + {\tt Out}(y) \\ 
			& \leqslant {\tt Opt}(\bm z, \lambda) + \mathcal{E}  [ \pi_{B} ( f_{\bm{\theta}^{\star}} ) ]   - \mathcal{E}_{\bm z}  [ \pi_{B} ( f_{\bm{\theta}^{\star}} ) ]  - \mathcal{E}(f_{\rho})  +  \mathcal{E}_{\bm z}  ( f^{\lambda}_{\bm \theta} ) + \lambda \| f^{\lambda}_{\bm \theta} \|_{\mathcal{P}} + {\tt Out}(y)
			\\	
			& = {\tt Opt}(\bm z, \lambda) + {\tt D}(\lambda) +  {\tt Out}(y) +  \mathcal{E}  [ \pi_{B} ( f_{\bm{\theta}^{\star}} ) ]   - \mathcal{E}_{\bm z}  [ \pi_{B} ( f_{\bm{\theta}^{\star}} ) ]  +  \mathcal{E}_{\bm z}  ( f^{\lambda}_{\bm \theta} ) - \mathcal{E}  ( f^{\lambda}_{\bm \theta} ) \\
			& =   {\tt Opt}(\bm z, \lambda) + {\tt  Out}  (y) + {\tt D}(\lambda) +  {\tt S}_1(\bm z, \lambda, \bm \theta) +  {\tt S}_2(\bm z, \lambda, \bm \theta)  \,,
		\end{split}
	\end{equation*}
	where the second inequality uses~\cref{outputy} for the output error, and the third inequality holds by the condition that ${f_{\bm \theta^{\star}}}$ is a global minimizer.
\end{proof}

\subsection{Proof of the optimization error}
\label{app:opterror}

There are several off-the-shelf algorithms to solve the convex optimization problem in \cref{fzfinitecom}, e.g., proximal-proximal gradient algorithm \citep{ryu2019proximal},
operator splitting \citep{mishchenko2019stochastic,salim2020dualize,davis2017three}, which is able to enjoy $\| \bm u^{(T)}_i - \bm u_i^{\star}\|^2_2 \lesssim \mathcal{O}(1/T^2)$ for strongly convex problem in~\cref{fzfinitecom} (with additional $\ell_2$ regularization), where $\bm u^{(T)}_i$ is the solution after $T$-th iterations.
Accordingly, we can transform convergence on sequence in optimization to the expected risk of empirical functional.

\begin{proposition}\label{propopt} [Optimization error]
	Denote the over-parameterized, two-layer ReLU neural network $f_{\bm \theta^{(T)}} (\bm x)= \frac{1}{m} \sum_{k=1}^m a_k^{(T)} \sigma( \langle \bm w_k^{(T)}, \bm x \rangle )  $ corresponding to \cref{fzfinitecom} with $\mathcal{O}(d r(n/r)^r)$ variables and $\mathcal{O}(n r(n/r)^r)$ linear inequalities with $r=\mbox{rank}(\bm X)$ solved by some convex optimization algorithms after $T$ iterations. Under Assumptions~\ref{assbounddata}, the expected risk of the optimal solution $f^*_{\bm z, \lambda} (\bm x)$ and the numerical solution $f_{\bm \theta^{(T)}} (\bm x)$ can be upper bounded by
	\begin{equation*}
		\mathcal{E} [ \pi_{B} ( f_{\bm{\theta}^{(T)}} ) ]   - \mathcal{E} [\pi_{B} ( f_{\bm{\theta}^{\star}} ) ]  \lesssim \mathcal{O} \left( \frac{1}{T^2} \right) \,.
	\end{equation*}
\end{proposition}

\begin{proof}
	The optimization error under deterministic optimization algorithms can be estimated by
	\begin{equation*}
		\begin{split}
			\mathcal{E} [ \pi_{B} ( f_{\bm{\theta}^{(T)}} ) ] - \mathcal{E} [\pi_{B} ( f_{\bm{\theta}^{\star}} ) ] & = \left\| \pi_{B} ( f_{\bm \theta^{(T)}} ) - \pi_{B}(f_{\bm \theta^{\star}}) \right\|^2_{L^2_{\rho_X}}  \\
			& =  \mathbb{E}_{\bm x} \left| \pi_{B} ( f_{\bm \theta^{(T)}}(\bm x) ) - \pi_{B} ( f_{\bm \theta^{\star}}(\bm x) )   \right|^2   \leqslant  \mathbb{E}_{\bm x} \left|  f_{\bm \theta^{(T)}}(\bm x)  -  f_{\bm \theta^{\star}}(\bm x)    \right|^2 \\
			& =  \mathbb{E}_{\bm x} \left| \frac{1}{m} \sum_{k=1}^m  \sigma( \langle \bm u_k^{(T)}, \bm x \rangle ) - \frac{1}{m} \sum_{k=1}^m  \sigma( \langle \bm u_k^{\star}, \bm x \rangle ) \right|^2 \,,
		\end{split}
	\end{equation*}
	where the inequality holds by $\pi_B(a)-\pi_B(b) \leqslant |a-b|$ and the last equality benefits from the positive homogeneity of ReLU.
	Furthermore, the above equation can be upper bounded by
	\begin{equation*}
		\begin{split}
			\mathcal{E} [ \pi_{B} ( f_{\bm{\theta}^{(T)}} ) ] - \mathcal{E} [\pi_{B} ( f_{\bm{\theta}^{\star}} ) ] & \leqslant  \mathbb{E}_{\bm x} \left| \frac{1}{m} \sum_{k=1}^m ( \bm u_k^{(T)} - \bm u_k^{\star})^{\!\top} \bm x  \right|^2  \leqslant  \mathbb{E}_{\bm x} \left| \frac{1}{m} \sum_{k=1}^m \|  \bm u_k^{(T)} - \bm u_k^{\star} \|_2 \| \bm x \|_2  \right|^2 \\
			& \leqslant   \frac{1}{m} \sum_{k=1}^m \|  \bm u_k^{(T)} - \bm u_k^{\star} \|^2_2  \mathbb{E}_{\bm x}\| \bm x \|_2^2 \\
			& \lesssim \mathcal{O}\left( \frac{1}{T^2} \right) \,,
		\end{split}
	\end{equation*}
	where the first inequality holds by 1-Lipschitz continuous of ReLU and the last inequality uses the convergence on sequence from optimization $\| \bm u^{(T)} - \bm u^{\star}\|_2^2 \lesssim \mathcal{O}(1/T^2)$, e.g., \citep{mishchenko2019stochastic,davis2017three}.
	Accordingly, we finish the proof.
\end{proof}



When employing stochastic approximation algorithms, e.g., stochastic decoupling algorithm with time-varying stepsize \citep{mishchenko2019stochastic} to solve problem~(\ref{fzfinitecomthree}), we have $\mathbb{E}_{\mathcal{J}}\| \bm u_i^{(T)} - \bm u_i^{\star} \|_2^2 \lesssim \mathcal{O} \left( \frac{1}{T^2} \right)$, where the randomness stems from sampling $i$ from the set $\mathcal{J} = \{ 1,2, \dots, 2P \}$ to efficiently conduct the proximal operation.
In this case, the optimization error under stochastic approximation algorithms can be still estimated in the similar way as below.
\begin{equation*}
	\begin{split}
		\mathbb{E}_{\mathcal{J}} \big( \mathcal{E} [ \pi_{B} ( f_{\bm{\theta}^{(T)}} ) ] \big)   - \mathcal{E} [\pi_{B} ( f_{\bm{\theta}^{\star}} ) ] & =  \mathbb{E}_{\mathcal{J}} \Big\{  \mathcal{E} [ \pi_{B} ( f_{\bm{\theta}^{(T)}} ) ]   - \mathcal{E} [\pi_{B} ( f_{\bm{\theta}^{\star}} ) ]  \Big\} = \mathbb{E}_{\mathcal{J}} \left\| \pi_{B} ( f_{\bm \theta^{(T)}} ) - \pi_{B}(f_{\bm \theta^{\star}}) \right\|^2_{L^2_{\rho_X}}  \\
		& \leqslant \mathbb{E}_{\mathcal{J}} \mathbb{E}_{\bm x} \left|  f_{\bm \theta^{(T)}}(\bm x)  -  f_{\bm \theta^{\star}}(\bm x)    \right|^2 \\
		& = \mathbb{E}_{\mathcal{J}} \mathbb{E}_{\bm x} \left| \frac{1}{m} \sum_{k=1}^m  \sigma( \langle \bm u_k^{(T)}, \bm x \rangle ) - \frac{1}{m} \sum_{k=1}^m  \sigma( \langle \bm u_k^{\star}, \bm x \rangle ) \right|^2 \,.
	\end{split}
\end{equation*}
Further, by virtue of 1-Lipschitz continuous of ReLU and Assumption~\ref{assbounddata}, we have
\begin{equation*}
	\begin{split}
			\mathbb{E}_{\mathcal{J}} \big( \mathcal{E} [ \pi_{B} ( f_{\bm{\theta}^{(T)}} ) ] \big)   - \mathcal{E} [\pi_{B} ( f_{\bm{\theta}^{\star}} ) ]  & \leqslant \mathbb{E}_{\mathcal{J}} \mathbb{E}_{\bm x} \left| \frac{1}{m} \sum_{k=1}^m ( \bm u_k^{(T)} - \bm u_k^{\star})^{\!\top} \bm x  \right|^2  \\
			& \leqslant   \frac{1}{m} \sum_{k=1}^m \mathbb{E}_{\mathcal{J}}\|  \bm u_k^{(T)} - \bm u_k^{\star} \|^2_2  \mathbb{E}_{\bm x}\| \bm x \|_2^2  \\
			& \lesssim \mathcal{O}\left( \frac{1}{T^2} \right)\,.
	\end{split}
\end{equation*}

\subsection{Proof of the output error}
\label{app:outputerr}


Estimation on the output error ${\tt Out}(y) := \frac{1}{n} \sum_{i=1}^n \left| \pi_{{B}}(y_i) - y_i \right|^2$ is one key result in our proof, which is significantly different from \citep{guo2013concentration,Shi2014Quantile,liu2021gen} in formulation and techniques.
They focus on $\frac{1}{n} \sum_{i=1}^n \left| \pi_{{B}}(y_i) - y_i \right|$ based on moment hypothesis for high moments, and thus leave an extra term difficult to converge to zero. 
Instead, we focus on the output error ${\tt Out}(y) := \frac{1}{n} \sum_{i=1}^n \left| \pi_{{B}}(y_i) - y_i \right|^2$, which is more intractable due to the squared order. 
In this case, the random variable $(|y_i| - B)^2 \mathbb{I}_{\{ |y_i| \geqslant B \} } $ is no longer sub-exponential but still admits the exponential-type decay. 
We introduce sub-Weibull random variables \citep{vladimirova2020sub,zhang2020concentration} to tackle this issue. The extra term in our result is $B\exp(-B)$ in an exponential decaying order. 
Accordingly, our result is able to work in a non-asymptotic regime as it does not require the exponential sample complexity.

To bound the output error, we need the following lemma.
Note that, the results presented here are also needed for the sample error to tackle the unbounded outputs. 
\begin{lemma}\label{lemmaboundy}
	Let $B \geqslant 1$ and $CM \geqslant 1$, under the moment hypothesis in \cref{Momenthypothesis}, the error bound on truncated outputs can be estimated by
	\begin{equation*}
		\int_{|y| \geqslant B} |y| \mathrm{d} \rho_{Y} \leqslant 2(B+4CM^2) \exp \left( -\frac{B}{4CM^2} \right)\,.
	\end{equation*}
\end{lemma}
\begin{proof}
Using the integral expectation formula $\mathbb{E}(X) = \int_{0}^{\infty} \mathrm{Pr} (X \geqslant t) \mathrm{d} t$ for any non-negative random variable $X$, we have
	\begin{equation}\label{boundy}
		\begin{split}
			\mathbb{E} \left[ |y| \mathbb{I}_{\{ |y| \geqslant B \}} \right] &= \mathbb{E} \left[ \int_{0}^{ |y|} 1 \mathrm{d} t \mathbb{I}_{\{ |y| \geqslant B \}} \right] = \mathbb{E} \left[ \int_0^{\infty} \mathbb{I}_{ \{  |y| \geqslant t \} } \mathbb{I}_{\{ |y| \geqslant B \}} \mathrm{d} t \right] \\
			& = \int_0^{\infty} \mathrm{Pr} \left[ |y| \geqslant t ,\; |Y| \geqslant B \right] \mathrm{d} t = \int_0^{\infty} \mathrm{Pr} \left[ |y| \geqslant \max\{ B, t \} \right] \mathrm{d} t \\
			& = \int_0^B \mathrm{Pr}  \left[ |y| \geqslant B \right] \mathrm{d} t + \int_B^{\infty} \mathrm{Pr}  \left[ |y| \geqslant t \right] \mathrm{d} t \\
			& = B \mathrm{Pr}  \left[ |y| \geqslant B \right]  + \int_B^{\infty} \mathrm{Pr}  \left[ |y| \geqslant t \right] \mathrm{d} t \\
			& \leqslant 2B  \exp \left( -\frac{B}{4CM^2} \right) + 8CM^2  \exp \left( -\frac{B}{4CM^2} \right) \,,
		\end{split}
	\end{equation}
	where the inequality holds for the moment hypothesis in \cref{Momenthypothesis}, which implies the sub-exponential property of $|y|$ when $y$ is zero-mean, i.e., $\mathrm{Pr}(|y| \geqslant t) \leqslant 2 \exp \left( -\frac{t^2}{ 2(CM^2 + Mt) } \right) $ and we use the fact $\frac{t^2}{2(CM^2+ Mt)} \geqslant \frac{t}{4CM^2}$ when $t \geqslant B \geqslant 1$ and $CM \geqslant 1$.
\end{proof}

Lemma~\ref{lemmaboundy} shows that the truncated outputs admit an exponential decay with the threshold $B$.
Now we are ready to prove Proposition~\ref{propoutput} for the output error.
\begin{proof}[Proof of Proposition~\ref{propoutput}]
	It is clear that
	\begin{equation*}
		| y - \pi_{B}(y) | = (|y| - B) \mathbb{I}_{\{ |y| \geqslant B \} }\,, \quad \mbox{and}~ | y - \pi_{B}(y) |^2 = (|y| - B)^2 \mathbb{I}_{\{ |y| \geqslant B \} }\,,
	\end{equation*}
	and thus we set a random variable $\zeta := (|y| - B) \mathbb{I}_{\{ |y| \geqslant B \} }$ on $(Z, \rho)$.
	Similar to \cref{boundy}, we have
	\begin{equation}\label{boundy2}
		\begin{split}
			\int_{|y| \geqslant B} ( |y| - B)^{p} d \rho & = \mathbb{E} \left[ \int_{0}^{( |y| - B)^{p} \mathbb{I}_{ \{  |y| \geqslant B \} } } 1 \mathrm{d} t \right] = \mathbb{E} \left[ \int_{0}^{\infty} \mathbb{I}_{ \{  ( |y| - B)^{p} \mathbb{I}_{ \{  |y| \geqslant B \} }  \}  } \mathrm{d}t   \right] \\
			& = \int_0^{\infty} \mathrm{Pr} \left[ ( |y| - B)^{p} \mathbb{I}_{ \{  |y| \geqslant B \} } \geqslant t \right] \mathrm{d} t \\
			& = \int_0^{\infty} \mathrm{Pr} \left[ ( |y| - B)^{p} \mathbb{I}_{ \{  |y| \geqslant B \} } \geqslant u \right] p u^{p - 1} \mathrm{d} u \\
			& =  \int_0^{\infty} \mathrm{Pr} \left[ |y|  \geqslant B + u  \right]  p u^{p - 1} \mathrm{d} u \\
			& \leqslant  2 \int_0^{\infty} \exp \left( -\frac{B+u}{4CM^2} \right) p u^{p - 1} \mathrm{d} u \\
			& = 2 \exp \left( -\frac{B}{4CM^2} \right) (4CM^2)^{p} p \int_0^{\infty} \exp(-t) t^{p - 1}  \mathrm{d} t \\
			& = 2\exp \left( -\frac{B}{4CM^2} \right) (4CM^2)^{p} p !  \,,
		\end{split}
	\end{equation}
	where the last equality follows with the expression of the Gamma function.
	Accordingly, we deduce that the random variable $\zeta$ is sub-exponential as 
	\begin{equation*}
		( \mathbb{E} \zeta^p )^{1/p} \leqslant  \left[ 2\exp \Big( -\frac{B}{4CM^2} \Big) \right]^{1/p} (4CM^2) (p !)^{1/p}  \leqslant \widetilde{C} p^{1+ \frac{1}{2p}} \rightarrow \mathcal{O}(p)  \quad \mbox{as} ~~ p \rightarrow \infty \,,
	\end{equation*}
	where $\widetilde{C}$ is some constant depending on $B, C, M$ and we use Stirling's approximation for factorials.

	
	Based on the above discussion, denote the random variable $\nu_i := (|y_i| - B)^2 \mathbb{I}_{\{ |y_i| \geqslant B \} } $, it is clear that $\nu_i$ is not sub-exponential but still admits the exponential-type decay. This is in fact a sub-Weibull random variable \citep{vladimirova2020sub,zhang2020concentration}.
	Precisely, a random variable $X$ satisfies $\mathrm{Pr}(|X| \geqslant t) \leqslant a \exp(-bt^\theta)$ for given $a, b, \theta$, denoted as $X \sim \mathrm{subW}(\theta)$.
	We can also use Orlicz-type norms for definition. To be specific, the sub-Weibull norm is defined as
	\begin{equation*}
	\| X \|_{\psi_{\theta}} := \inf \left\{ c \in (0, \infty): \mathbb{E} [\exp(|X|^\theta/ c^{\theta})] \leqslant 2 \right\}\,.    
	\end{equation*}
	In this case, $X$ is a sub-Weibull random variable if it has a bounded $\psi_{\theta}$-norm. In particular, if we take $\theta = 1$, we get the sub-exponential norm.
	Obviously, the random variable $\nu_i$ is sub-Weibull and $\theta = 1/2$.
	Using concentration for sub-Weilbull summation in \cite[Theorem 1]{zhang2021sharper}, we have
	\begin{equation}
		\begin{split}
			\mathrm{Pr}\left(\left| \frac{1}{n} \sum_{i=1}^{n} \nu_{i} - \mathbb{E} \nu \right| \geqslant s \right) & \leqslant 2 \exp \left\{-\left(\frac{s^{\theta}}{\left[4 e C(\theta)\|\boldsymbol{b}\|_{2} L_{n}(\theta, \boldsymbol{b})\right]^{\theta}} \wedge \frac{s^{2}}{16 e^{2} C^{2}(\theta)\|\boldsymbol{b}\|_{2}^{2}}\right)\right\}  \\
			& = 	2 e^{-s^{\theta} /\left[4 e C(\theta)\| \bm b\|_{2} L_{n}(\theta, \bm b)\right]^{\theta}}, \text { if } s>4 e C(\theta)\|\bm{b}\|_{2} L_{n}^{\theta /(\theta-2)}(\theta, \bm{b}) \,,
		\end{split}
	\end{equation}
	where $\bm b = \frac{1}{n} [ \| \nu_1 \|_{\psi_{\theta}},  \| \nu_2 \|_{\psi_{\theta}}, \cdots, \| \nu_n \|_{\psi_{\theta}}]^{\!\top} = \frac{\| \bm \nu \|_{\psi_{\theta}}}{n} \bm 1 $, and $C(\theta)$ is defined as
	\begin{equation*}
		C(\theta) := 2\left[\log ^{1 / \theta} 2+e^{3}\left(\Gamma^{1 / 2}\left(\frac{2}{\theta}+1\right)+3^{\frac{2-\theta}{3 \theta}} \sup _{p \geqslant 2} p^{-\frac{1}{\theta}} \Gamma^{1 / p}\left(\frac{p}{\theta}+1\right)\right)\right] \,,
	\end{equation*}
	and $L_n(\theta, \bm b) = \gamma^{2/\theta} A(\theta) \frac{\| \bm b \|_{\infty}}{\| \bm b \|_2}$ with
	\begin{equation*}
		A(\theta)=: \inf _{p \geqslant 2} \frac{e^{3} 3^{\frac{2-\theta}{3 \theta}} p^{-1 / \theta} \Gamma^{1 / p}\left(\frac{p}{\theta}+1\right)}{2\left[\log ^{1 / \theta} 2+e^{3}\left(\Gamma^{1 / 2}\left(\frac{2}{\theta}+1\right)+3^{\frac{2-\theta}{3 \theta}} \sup _{p \geqslant 2} p^{-\frac{1}{\theta}} \Gamma^{1 / p}\left(\frac{p}{\theta}+1\right)\right)\right]} \,,
	\end{equation*}
	and $\gamma$ is the smallest solution of the equality $ \{ k > 1: e^{2 k^{-2}}-1+\frac{e^{2\left(1-k^{2}\right) / k^{2}}}{k^{2}-1} \leqslant 1 \}$. An approximate solution is $\gamma \approx 1.78$.
	
	By elementary calculation, we have $C(\theta) = 2 [\log^2 2 + 2(e^3 + 3/4) \sqrt{6}]$ and $L_n(\theta, \bm b) = \gamma^4 \frac{4e^3 \sqrt{6}}{\log^2 2 + 2(e^3 + 3/4) \sqrt{6}} n^{-1/2}$ by taking $\theta = 1/2$. If $s \geqslant 8e [\log^2 2 + 2(e^3 + 3/4) \sqrt{6}]^{\frac{4}{3} } \| \nu \|_{\psi_{\frac{1}{2}}} n^{-1/3} = \mathcal{O}(n^{-1/3})$, we have 
	\begin{equation*}
		\mathrm{Pr}\left(\left| \frac{1}{n} \sum_{i=1}^{n} \nu_{i} - \mathbb{E} \nu \right| \geqslant s \right)  \leqslant  \exp \left( - \frac{(ns)^{1/2}}{\widetilde{C}} \right) \,,
	\end{equation*}
	where $\widetilde{C}$ is some constant independent of $n$.
	Note that the condition $s > cn^{-1/3}$ for some cosntant $c$ is fair as $n$ is large in practice.
	Setting the right-hand side to be ${\delta}/{4}$, we deduce that with probability at least $1 - {\delta}/{4}$, there holds
	\begin{equation*}
		\frac{1}{n} \sum_{i=1}^{n} \nu_{i}  \lesssim \mathbb{E}(\nu) + \frac{1}{n} \log^{2} \frac{4}{\delta} \,,
	\end{equation*}
	which implies
	\begin{equation*}
		\frac{1}{n} \sum_{i=1}^{n} \left( \left|y_{i}\right| - B \right)^2 \mathbb{I}_{ \{ \left|y_{i}\right| \geqslant B \} } \lesssim  \exp \left( -\frac{B}{4CM^2} \right) (4CM^2)^{2}   + \frac{1}{n} \log^{2} \frac{4}{\delta} \,.
	\end{equation*}
	Finally, we conclude the proof.
\end{proof}

\subsection{Proof of the sample error}
\label{app:sampleerr}


In this subsection, we give the proofs for estimating the sample error.

\subsubsection{Proof of  ${\tt S_2}$ in the sample error}
\label{app:S2}
This part for estimation on ${\tt S_2}$ is similar to \citep{Wang2011Optimal}  apart from the studied function space is different.
For completeness of proof, we include the proof here for self-completeness.

In our problem, we set $\left\{\xi_{i}\right\}_{i=1}^n$ to be independent random variables on
$(Z,\rho)$ with $\bm z= (\bm x, y)$ defined as 
\begin{equation*}\label{randomvariables}
	\xi(\bm x, y) := \big(y-f^{\lambda}_{\bm \theta}(\bm x)\big)^2 - \big(y-f_{\rho}(\bm x)\big)^2\,,
\end{equation*}
such that $A_i = \mathbb{E}\xi - \xi(\bm z_i)$ for the function $\xi$ and $\mathbb{E} \xi = \int_Z \xi(\bm z) \mathrm{d} \rho$.

\begin{proposition}\label{props2}
	Under Assumptions~\ref{assbounddata},~\ref{assrho} and moment hypothesis in \cref{Momenthypothesis}, there exists a subset of $Z_2$ of $Z^{n}$ with confidence at least $1-\delta/4$ with $0<\delta<1$, such that $\forall {\bm z}:=(\bm x, y) \in Z_2$
	\begin{equation*}\label{bound7}
		\begin{split}
			{\tt S}_2(\bm z, \lambda, \bm \theta) 	\lesssim (CM)^2 \left( \frac{S}{n} \log \frac{4}{\delta} + \lambda \right) \,.
		\end{split}
	\end{equation*}
\end{proposition}
\begin{proof}
	According to the definition of the considered function space $\mathcal{P}_m$, we have
	\begin{equation}\label{fbound}
		|f_{\bm \theta}(\bm x)| = \frac{1}{m}\sum_{k=1}^m |a_k| \max(\bm \omega_k^{\!\top} \bm x, 0) \leqslant \frac{1}{m} \sum_{k=1}^m |a_k| |\bm w_k^{\!\top} \bm x| \leqslant \frac{1}{m} \sum_{k=1}^m |a_k| |\bm w_k \|_1 \| \bm x \|_{\infty} = \| \bm x \|_{\infty} \| \bm \theta \|_{\mathcal{P}} \,,
	\end{equation}
	where the last inequality holds by H\"{o}lder inequality.
	Accordingly, we have
	\begin{equation}\label{flambound}
		\| f^{\lambda}_{\bm \theta}\|_{\infty} \leqslant  \| \bm x \|_{\infty} \| \bm \theta^{\lambda} \|_{\mathcal{P}} \leqslant  \| \bm x \|_{\infty} \frac{{\tt D}(\lambda)}{\lambda} \leqslant CM \| \bm x \|_{\infty} \,,
	\end{equation}
 where $\bm \theta^{\lambda}$ denotes the parameter of $f_{\bm \theta}^{\lambda}$.
The last two inequalities hold by the definition of $f^{\lambda}_{\bm \theta}$ in \cref{flamdadef} and ${\tt D}(\lambda)$ in \cref{Dlamdadef}, respectively.

	By elementary inequalities, we have
	\begin{equation*}
		\begin{aligned}
			|\xi(\bm z)|^{p} &=\left|f_{\rho}(\bm x)-f^{\lambda}_{\bm \theta}(\bm x)\right|^{p}\left|f_{\rho}(\bm x) + f^{\lambda}_{\bm \theta}(\bm x)-2 y\right|^{p} \\
			& \leq 3^{p}\left(\left|f^{\lambda}_{\bm \theta}(\bm x) \right|^{p}+\left| f_{\rho}(\bm x)\right|^{p}+2^{\ell}|y|^{p}\right) 2^{\ell-2}\left(\left|f^{\lambda}_{\bm \theta}(\bm x)\right|^{p-2}+\left|f_{\rho}(\bm x)\right|^{p -2}\right)\left|f^{\lambda}_{\bm \theta}(\bm x)-f_{\rho}(\bm x)\right|^{2} \,,
		\end{aligned}
	\end{equation*}
	which implies 
	\begin{equation*}
		\begin{split}
			\mathbb{E} 	|\xi(\bm z)|^{p} & = \int_X \int_Y |\xi(\bm z)|^{p} \mathrm{d} \rho(y | \bm x) \mathrm{d} \rho_X(\bm x) \\
			& \leqslant 3^{p}\left(\left|f^{\lambda}_{\bm \theta}(\bm x) \right|^{p}+\left| f_{\rho}(\bm x)\right|^{p}+2^{\ell}|y|^{p}\right) 2^{\ell-2}\left(\left|f^{\lambda}_{\bm \theta}(\bm x)\right|^{p-2}+\left|f_{\rho}(\bm x)\right|^{p -2}\right) \int_X [f^{\lambda}_{\bm \theta}(\bm x ) - f_{\rho}(\bm x)]^2 \mathbb{d} \rho_X(\bm x) \\
			& \leqslant  p ! 6^{p} \left( (CM + 1)^2  \| \bm x \|_{\infty} + 2(C+1)^2M^2 \right)^{p-1} {\tt D}(\lambda) \,,
		\end{split}
	\end{equation*}
	where the last inequality holds by \cref{flambound}.
	Accordingly, we have $\mathbb{E}|\xi - \mathbb{E} \xi|^{p} \leqslant 2^{p} \mathbb{E}|\xi|^{p}$.
	
	Denoting $\widetilde{M} := 12(CM + 1)^2  \| \bm x \|_{\infty} + 2(C+1)^2M^2 $ and $\widetilde{v} := 576 \widetilde{M} {\tt D}(\lambda)$ such that
	$\mathbb{E}|\xi - \mathbb{E} \xi|^{p} \leqslant \frac{1}{2} p \widetilde{M}^{p -2} \widetilde{v} $, then by Lemma~\ref{lemaunbounded}, we have
	\begin{equation*}
		\mathrm{Prob}_{\bm{z} \in Z^{n}}\left\{\mathbb{E} \xi-\mathbb{E}_{\bm{z}} \xi \geqslant \epsilon \right\} \leqslant \exp \left\{-\frac{n \epsilon^{2}}{2\left( \widetilde{v} + \widetilde{M} \epsilon\right)} \right\} .
	\end{equation*}
	By setting the right-hand side to be $\delta/4$ such that
	\begin{equation}
		\begin{aligned}
			\epsilon &=\frac{1}{n}\left\{\widetilde{M} \log \frac{4}{\delta}+\sqrt{\widetilde{M}^{2} \log^{2} \frac{4}{\delta}+2 n \widetilde{v} \log \frac{4}{\delta} } \right\} \leqslant \frac{20 \widetilde{M}}{n} \log \frac{4}{\delta}+18 {\tt D}(\lambda) \,.
		\end{aligned}
	\end{equation}
	Then there exists a subset $Z_2$ of $Z^n$ with confidence $1-\delta/4$ such that
	\begin{equation*}
		\begin{split}
			\mathbb{E} \xi - \mathbb{E}_{\bm z} \xi \leqslant \frac{20 \widetilde{M}}{n} \log \frac{4}{\delta}+18 {\tt D}(\lambda) \lesssim 
			(CM)^2 \left( \frac{\| \bm x \|_{\infty}}{n}\log \frac{4}{\delta}  + \lambda \right) \,,
		\end{split}
	\end{equation*}
	and thus we conclude the proof.
\end{proof}

\subsubsection{Proof of ${\tt S_1}$ in the sample error}
\label{app:S1}

Estimation for ${\tt S_1}$ is more challenging than ${\tt S_2}$ as the random variable $(f_{\bm \theta}(\bm x) - y)^2 - (f_{\rho}(\bm x) - y)^2$ is also depends on the sample $\bm z$ itself, in which the result on the metric entropy is needed.
This is also one key technical result in our proof framework, which needs the results of the truncated output error in Appendix~\ref{app:outputerr} as well.
\begin{proposition}\label{props1}
	Under Assumptions~\ref{assbounddata},~\ref{assrho}, and moment hypothesis in~\cref{Momenthypothesis} let $R \geqslant B \geqslant M \geqslant 1$ and $CM \geqslant 1$, there exists a subset of $Z' $ of $Z^{n}$ with confidence at least $1-\delta/2$ with $0 < \delta < 1$ such that for any $\bm z:=(\bm x, y) \in Z'$ and $f_{\bm \theta^{\star}} \in \mathcal{G}_R$
	\begin{equation*}
		\begin{split}
			  {\tt S_1}  & \leqslant \frac{1}{2}  \Big\{	\mathcal{E}\big[ \pi_B (f_{\bm \theta^{\star}}) \big] -  \mathcal{E}(f_{\rho})  \Big\} +  \widetilde{C}Rd n^{-\frac{d+2}{2d+2}}  \log \frac{4}{\delta} \!+\! 4(B+CM) (B+ 72CM^2) \exp \Big(\! -\frac{B}{4CM^2} \!\Big) \\
			& \quad + \frac{\widetilde{C}}{n} \log \frac{4}{\delta} \,,
		\end{split}
	\end{equation*}
	where $\widetilde{C} $ is some constant depending on $B, C, M, c_q, c_q'$ but independent of $n, \delta$, and $d$.
\end{proposition}

To prove Proposition~\ref{props1}, we need to the following lemma on concentration of truncated outputs as below.
\begin{lemma}\label{lemmaboundy2}
	Let $B \geqslant 1$ and $CM \geqslant 1$, under the moment hypothesis in \cref{Momenthypothesis}, there exists a subset $Z_3$ of $Z^n$ with probability at least $1 - \frac{\delta}{4}$ such that  $\forall \bm z:=(\bm x, y) \in Z_{3}$
	\begin{equation*}
		\begin{aligned}
			&\frac{1}{n} \sum_{i=1}^{n} \left( \left|y_{i}\right| - B \right) \mathbb{I}_{ \{ \left|y_{i}\right| \geqslant B \} } - \int_{Z} \left( \left|y_{i}\right| - B \right) \mathbb{I}_{ \{ \left|y_{i}\right| \geqslant B \} } \mathrm{d} \rho \leqslant \frac{24CM^2}{n} \log \frac{4}{\delta} + 128CM^2 \exp\left(-\frac{B}{CM^2} \right) \,.
		\end{aligned}
	\end{equation*}
\end{lemma}
\begin{proof}
	Define a random variable $\zeta := (|y| - B) \mathbb{I}_{\{ |y| \geqslant B \} }$ on $(Z, \rho)$, similar to \cref{boundy2}, we have
	\begin{equation*}
		\begin{split}
			\mathbb{E}|\zeta-\mathbb{E} \zeta|^{p} & \leqslant 2^{p+1} \mathbb{E}|\zeta|^{p} = 2^{p+1} \int_{|y| \geqslant B} ( |y| - B)^{p} \mathbb{I}_{ \{  |y| > B \} } d \rho  = 2^{p+1}  \int_{|y| \geqslant B} ( |y| - B)^{p} d \rho \\
			& 	\leqslant 2^{p+2} \exp \left( -\frac{B}{4CM^2} \right) (4CM^2)^{p} p !  = \frac{1}{2} p! \widetilde{M} ^{p-2} \widetilde{v}  \,,
		\end{split}
	\end{equation*}
	where $\widetilde{M}  := 	8CM^2$ and $\widetilde{v} := 8\widetilde{M}^2  \exp \left( -\frac{B}{4CM^2} \right)  $.
	
	By Lemma~\ref{lemaunbounded}, setting the right-hand side of \cref{leamm4eq} to be $\delta/4$ such that
	\begin{equation*}
		\begin{aligned}
			\varepsilon &= \widetilde{M} \log \frac{4}{\delta}+\sqrt{\widetilde{M}^{2} \log^{2} \frac{4}{\delta}+2 n \tilde{v} \log \frac{4}{\delta} }  \leqslant  3 \widetilde{M}  \log \frac{4}{\delta} + 128CM^2 n \exp \left( -\frac{B}{CM^2} \right) \,,
		\end{aligned}
	\end{equation*}
	which concludes the proof.
\end{proof}



Our error analysis on ${\tt S}_1$ replies on the following concentration inequality which can be found in \citep{blanchard2008statistical}.
Before introducing this, we need the definition of \emph{sub-root} function. 
\begin{definition}
	A function $\psi: \mathbb{R}_+ \rightarrow \mathbb{R}_+$ is sub-root if it is non-negative, non-decreasing, and if $\psi(x)/\sqrt{x}$ is non-increasing.
\end{definition}
It is easy to see for a sub-root function and any $D > 0$, the equation  $\psi(r) = r/D$ has unique positive solution.

Now we are ready to prove Proposition~\ref{props1} for estimation of ${\tt S_1}$.

\begin{proof}[Proof of Proposition~\ref{props1}]
	Denote the empirical and expected risk on truncated output as
	\begin{equation*}
		\widetilde{\mathcal{E}}_{\bm z}(f_{\bm \theta}) = \frac{1}{n} \sum_{i=1}^n \left[ \pi_B(f_{\bm \theta}(\bm x_i)) - \pi_B(y_i) \right]^2 \,, \qquad \widetilde{\mathcal{E}}(f_{\bm \theta}) = \mathbb{E} [ \pi_B(f_{\bm \theta}(\bm x)) - \pi_B(y) ]^2 \,,
	\end{equation*}
	and recall the definition of $ {\tt S}_1(\bm z, \lambda, \bm \theta)$, we can further decompose it as
	\begin{align}
		 {\tt S}_1(\bm z, \lambda, \bm \theta) & = \left[	\mathcal{E}\big(\pi_B (f_{\bm \theta^{\star}}) \big) -  \mathcal{E}(f_{\rho}) \right] - \left[ \mathcal{E}_{\bm{z}}\big(\pi_B (f_{\bm \theta^{\star}})\big) - \mathcal{E}_{\bm{z}}(f_{\rho}) \right]  \label{YYS1}\\
		& =\left[	\mathcal{E}\big(\pi_B (f_{\bm \theta^{\star}})\big) -  \mathcal{E}(f_{\rho}) \right]  -  \left[  	\widetilde{\mathcal{E}}\big(\pi_B (f_{\bm \theta^{\star}})\big) -  \widetilde{\mathcal{E}}(f_{\rho})   \right] \tag{\ref{YYS1}{a}} \label{YYa}\\
		& \quad +  \left[  	\widetilde{\mathcal{E}}_{\bm z}\big(\pi_B (f_{\bm \theta^{\star}})\big) -  \widetilde{\mathcal{E}}_{\bm z}(f_{\rho})   \right] - \left[ \mathcal{E}_{\bm{z}}\big(\pi_B (f_{\bm \theta^{\star}})\big) - \mathcal{E}_{\bm{z}}(f_{\rho}) \right] \tag{\ref{YYS1}{b}} \label{YYb}\\
		& \quad  + \left[  	\widetilde{\mathcal{E}}\big(\pi_B (f_{\bm \theta^{\star}})\big) -  \widetilde{\mathcal{E}}(f_{\rho})   \right]  -  \left[  	\widetilde{\mathcal{E}}_{\bm z}\big(\pi_B (f_{\bm \theta^{\star}})\big) -  \widetilde{\mathcal{E}}_{\bm z}(f_{\rho})   \right]   \tag{\ref{YYS1}{c}} \label{YYc} \,.
	\end{align}
	We aim to estimate the above three parts, respectively.
	
\noindent	{\bf Estimation of \cref{YYa}:} ${\tt S_{11}}$ for short, it can be bounded by
	\begin{equation}\label{eqs11}
		\begin{split}
			{\tt S_{11}} & := \left[	\mathcal{E}\big(\pi_B (f_{\bm \theta^{\star}})\big) -  \mathcal{E}(f_{\rho}) \right]  -  \left[  	\widetilde{\mathcal{E}}\big(\pi_B (f_{\bm \theta^{\star}})\big) -  \widetilde{\mathcal{E}}(f_{\rho})   \right]  \\
			& = \int_Z \bigg\{ [\pi_B (f_{\bm \theta^{\star}} (\bm x) ) - y]^2 - [f_{\rho}(\bm x) - y]^2 - [\pi_B (f_{\bm \theta^{\star}} (\bm x) ) - \pi_B(y)]^2 +  [f_{\rho}(\bm x) - \pi_B(y) ]^2 \bigg\}   \mathrm{d} \rho \\
			& = 2 \int_Z [ \pi_B (f_{\bm \theta^{\star}} (\bm x) ) - f_{\rho}(\bm x)] [y - \pi_B(y)] \mathrm{d} \rho \\
			& \leqslant 2 (B + CM) \int_{ |y| \geqslant B } |y| \mathrm{d} \rho_Y \\
			& \leqslant 4 (B + CM) (B + 4CM^2)  \exp \left( -\frac{B}{4CM^2} \right) := \varOmega \,,
		\end{split}
	\end{equation}
	where the last inequality holds by Lemma~\ref{lemmaboundy}.
	
\noindent	{\bf Estimation of \cref{YYb}:} ${\tt S_{12}}$ for short, then for each $\bm z \in Z_3$, with confidence $1 - {\delta}/{4}$, it can be bounded by
	\begin{equation}\label{eqs12}
		\begin{split}
			{\tt S_{12}} & := \left[  	\widetilde{\mathcal{E}}_{\bm z}\big(\pi_B (f_{\bm \theta^{\star}})\big) -  \widetilde{\mathcal{E}}_{\bm z}(f_{\rho})   \right] - \left[ \mathcal{E}_{\bm{z}}\big(\pi_B (f_{\bm \theta^{\star}})\big) - \mathcal{E}_{\bm{z}}(f_{\rho}) \right]  \\
			& = \frac{1}{n} \sum_{i=1}^n \left\{ [\pi_B (f_{\bm \theta^{\star}}(\bm x_i)) - \pi_B(y_i)]^2 -  [f_{\rho}(\bm x_i) - \pi_B(y_i) ]^2 - [\pi_B (f_{\bm \theta^{\star}}(\bm x_i)) - y_i]^2 + [f_{\rho}(\bm x_i) - y_i]^2  \right\}  \\
			& = \frac{2}{n} \sum_{i=1}^n [ \pi_B (f_{\bm \theta^{\star}}(\bm x_i)) - f_{\rho}(\bm x_i)] [y_i - \pi_B(y_i)]  \\
			& \leqslant  \frac{ 2 (B + CM) }{n}\sum_{i=1}^n  (|y_i| - B) \mathbb{I}_{ \{ |y_i| \geqslant B \} } \\
			& \leqslant  2 (B + CM) \left[ \int_Y (|y| - B) \mathbb{I}_{\{ |y_i| \geqslant B \}} \mathrm{d} \rho  + \frac{24CM^2}{n} \log \frac{4}{\delta} + 128CM^2  \exp \Big( -\frac{B}{4CM^2} \Big)    \right] \\
			& \leqslant  2 (B + CM) \left[ \frac{24CM^2}{n} \log \frac{4}{\delta} + 136CM^2  \exp \Big( -\frac{B}{4CM^2} \Big) \right] \,,
		\end{split}
	\end{equation}
	where the second inequality holds by Lemma~\ref{lemmaboundy2} and the last inequality uses \cref{boundy2}.
	
\noindent	{\bf Estimation of \cref{YYc}:} ${\tt S_{13}}$ for short, it is in fact the gap between empirical and expected version of a function $g_{\pi}(\bm z) := \left( \pi_B(f_{\bm \theta}( \bm x)) -\pi_{B}(y)\right)^{2}-\left(f_{\rho}(\bm x)-\pi_{B}(y)\right)^{2}$, which can be bounded by Lemma~\ref{lemsubroot} and the metric entropy result. The proof is relatively complex and we also split it into the following four steps.
	
\noindent	{\bf Step 1: the metric entropy}\\
	Define the set $\mathcal{G}_{R, \pi} $ of measurable functions given by
	$\mathcal{G}_{R, \pi} := \{ g_{\pi}(\bm z) + \varOmega: f_{\bm \theta} \in \mathcal{G}_R  \}$.
	The metric entropy of this function class can be estimated as follows.
	For any two functions $g_{1, \pi} + \varOmega \in \mathcal{G}_{R, \pi}$ and $g_{2, \pi} + \varOmega \in \mathcal{G}_{R, \pi}$, and $\bm z :=(\bm x, y) \in Z$, we have
	\begin{equation*}
		\begin{split}
			\left| (g_{1, \pi} + \varOmega) - (g_{2, \pi} + \varOmega) \right| & =\left| \left( \pi_{B} [f_{1}(\bm x)] - \pi_{B}(y)\right)^{2} -\left( \pi_{B} [f_{2}(\bm x)] - \pi_{B}(y)\right)^{2}\right| \\
			& = \Big|   \pi_{B} [f_{1}(\bm x)] -   \pi_{B} [f_{2}(\bm x)] \Big|  \Big|  \pi_{B} [f_{1}(\bm x)] +  \pi_{B} [f_{1}(\bm x)] - 2\pi_{B}(y) \Big| \\
			& \leqslant 4B \left|  f_{1}(\bm x) -  f_{2}(\bm x) \right| \,,
		\end{split}
	\end{equation*}
	which implies
	\begin{equation}\label{capres1}
		\mathscr{N}_{2}\left(\mathcal{G}_{R, \pi}, \varepsilon\right) \leq \mathscr{N}_{2}\left(\mathcal{G}_{R}, \frac{\varepsilon}{4B} \right)=\mathscr{N}_{2}\left(\mathcal{G}_{1}, \frac{\varepsilon}{4BR}\right) \,.
	\end{equation}
	Accordingly, denoting $q:= \frac{2d}{d+2}$, the metric entropy result in Proposition~\ref{prop:coverq2} yields
	\begin{equation}\label{capres}
		\log 	\mathscr{N}_{2}\left(\mathcal{G}_{R, \pi}, \varepsilon\right) \leqslant c_q d R^q (4B)^q \varepsilon^{-q} \,,
	\end{equation}
 where $c_q$ is a universal constant independent of $n$ and $d$.
	
\noindent	{\bf Step 2: Finite second moment}\\
	In the next, we use Lemma~\ref{lemsubroot} and the metric entropy in \cref{capres} on $\mathcal{G}_{R, \pi} $ to find the sub-root function $\psi$ in our setting. To this end, let $\{ \xi_i \}_{i=1}^n$ with $\mathrm{Pr}(\xi_i = 1) =\mathrm{Pr}(\xi_i=-1) = 1/2 $ be iid Rademancher sequence, the second moment of $g_\pi$ exists, we have \cite[Lemma 2.3.1]{van1996weak}
	\begin{equation}\label{eimp}
		\mathbb{E}\left[\sup _{g_{\pi} \in \mathcal{G}_{R,\pi},~ \mathbb{E} g_\pi^{2} \leqslant r}\left|\mathbb{E} g-\frac{1}{n} \sum_{i=1}^{n} g_{\pi}\left(\bm z_{i}\right)\right|\right] \leqslant 2 \mathbb{E}\left[\sup _{g_\pi \in \mathcal{G}_{R,\pi}, \mathbb{E} g_{\pi}^{2} \leqslant r}\left|\frac{1}{n} \sum_{i=1}^{n} \xi_{i} g_{\pi} \left(\bm z_{i}\right)\right|\right] \,.
	\end{equation}
	To use this result, we need check the condition on uniform boundedness of $ g_{\pi}$ and its second moment. 
	For $	\| g_{\pi} \|_{\infty} $, we have
	\begin{equation*}
		\begin{split}
			\| g_{\pi} \|_{\infty} & = \sup_{\bm z \in Z} \Big| \big[ \pi_B(f_{\bm \theta}(\bm x)) - f_{\rho}(\bm x) \big] \big[ \pi_B(f_{\bm \theta}(\bm x)) + f_{\rho}(\bm x) - 2 \pi_B(y) \big] \Big| \\
			& \leqslant ( B + CM) ( 3B + CM) \,,
		\end{split}
	\end{equation*}
	which implies
	\begin{equation*}
		|g_{\pi}(\bm z) + \varOmega  | \leqslant (B+CM) \left[4(B+4CM^2) \exp \left( -\frac{B}{4CM^2} \right) + 3B+CM \right]  \,,
	\end{equation*}
	where we use \cref{eqs11}.
	For the second-order moment of $g_{\pi}(\bm z)$, we have
	\begin{equation*}
		\begin{split}
			\mathbb{E} [g_{\pi}]^2 & = \int_X \int_Y [ \pi_B(f_{\bm \theta}(\bm x)) - f_{\rho}(\bm x)]^2 [ \pi_B(f_{\bm \theta}(\bm x)) + f_{\rho}(\bm x) - 2 \pi_B(y)]^2 \mathrm{d} \rho(y|\bm x) \mathrm{d} \rho_X(\bm x) \\
			& \leqslant \left[ \mathcal{E}(\pi_B(f)) - \mathcal{E}(f_{\rho})  \right] (3B+CM)^2 \\	
			& \leqslant \mathbb{E}[ g_{\pi} + \varOmega ]  (3B+CM)^2 \,,
		\end{split}
	\end{equation*}
	where we use $\int_X [ \pi_B(f_{\bm \theta}(\bm x)) - f_{\rho}(\bm x)]^2 \mathrm{d} \rho_X = \mathcal{E}(\pi_B(f)) - \mathcal{E}(f_{\rho}) = \mathbb{E} g $, with $g(\bm z):= (\pi_B(f_{\bm \theta}(\bm x)) - y)^2 - (f_{\rho}(\bm x) - y)^2$.
	That means
	\begin{equation*}
		\begin{split}
			\mathbb{E} [ g_{\pi} + \varOmega ]^2 & = \mathbb{E} [g_{\pi}]^2 + 2 \varOmega \mathbb{E} [g_{\pi}] + \varOmega^2 \\
			& \leqslant \mathbb{E}[ g_{\pi} + \varOmega ]  (3B+CM) \left[ 3B + CM + 8 ( B + 4 CM^2 ) \exp\Big(-\frac{B}{4CM^2} \Big)  \right] \,.
		\end{split}
	\end{equation*}

\noindent	{\bf Step 3: Bound \cref{eimp}}\\
	Since we have already verified the finite second moment of $g_{\pi}$, in the next, we aim to bound the right-hand of \cref{eimp} by the $\ell_2$-empirical covering number.
	Since $f \mapsto \frac{1}{\sqrt{n}} \sum_{i=1}^n \xi_i f(\bm z_i)$ is a sub-Gaussian process, according to the chain argument \citep{gine2021mathematical}, there exists a universal constant $C$ such that
	\begin{equation}\label{eqb1}
		\begin{split}
			\frac{1}{\sqrt{n}} \mathbb{E}_{\xi} \sup_{g_{\pi} \in \mathcal{G}_{R,\pi}, \mathbb{E} g_{\pi}^{2} \leqslant r}\left|\sum_{i=1}^{n} \xi_{i} g_{\pi}\left(\bm z_{i}\right)\right| 
			& \leqslant  C \int_{0}^{\sqrt{V}} \sqrt{\log_{2} \mathscr{N}_{2}\left(\mathcal{G}_{R,\pi}, \nu\right) } \mathrm{d} \nu \\
			& \leqslant C \int_{0}^{\sqrt{V}} \sqrt{ \log_2 \mathscr{N}_{2}\left(\mathcal{G}_{1}, \frac{\nu}{4BR}\right) } \mathrm{d} \nu \\
			& \leqslant C d^{\frac{1}{2}} c^{1/2}_q (4BR)^{\frac{q}{2}} \int_{0}^{\sqrt{V}}  \nu^{-\frac{q}{2}} \mathrm{d} \nu \\
			& = c'_q d^{\frac{1}{2}} (4BR)^{\frac{q}{2}} \left[ \sup_{g_{\pi} \in \mathcal{G}_{R,\pi}, \mathbb{E} g_{\pi}^{2} \leqslant r} \frac{1}{n} \sum_{i=1}^n g_{\pi}^2(\bm z_i) \right]^{\frac{1}{2} - \frac{q}{4}}\,,
		\end{split} 
	\end{equation}
	where $V:= \sup_{g_{\pi} \in \mathcal{G}_{R,\pi}, \mathbb{E} g_{\pi}^{2} \leqslant r} \frac{1}{n} \sum_{i=1}^n g_{\pi}^2(\bm z_i)$ and we use the result in \cref{capres1} and \cref{capres} in {\bf Step 1}.
	By virtue of Talagrand’s concentration inequality \citep{talagrand1996new}
	\begin{equation*}\label{eqb2}
		\mathbb{E} \sup_{g_{\pi} \in \mathcal{G}_{R,\pi}, \mathbb{E} g_{\pi}^{2} \leqslant r} \frac{1}{n} \sum_{i=1}^n g_{\pi}^2(\bm z_i) \leqslant \frac{8B}{n} \mathbb{E} \mathbb{E}_{\xi} \sup_{g_{\pi} \in \mathcal{G}_{R,\pi}, \mathbb{E} g_{\pi}^{2} \leqslant r} \left| \sum_{i=1}^n \xi_i g_{\pi}(\bm z_i) \right| + r\,,
	\end{equation*}
	taking it back to \cref{eqb1}, we have
	\begin{equation*}
		\mathscr{A}_n := 	\frac{1}{\sqrt{n}} \mathbb{E} \mathbb{E}_{\xi} \sup_{g_{\pi} \in \mathcal{G}_{R,\pi}, \mathbb{E} g_{\pi}^{2} \leqslant r}\left|\sum_{i=1}^{n} \xi_{i} g_{\pi}\left(\bm z_{i}\right)\right| \leqslant \tilde{c}_q \left( \frac{B\mathscr{A}_n}{\sqrt{n}} + r \right)^{\frac{1}{2} - \frac{q}{4}} (4BR)^{\frac{q}{2}} d^{\frac{1}{2}}\,,
	\end{equation*}
	where we use the same notation on the constant $c'_q$ that may change in the following derivations for notational simplicity.
	After solving this inequality, and taking it back to \cref{eimp}, we have
	\begin{equation}\label{eqslemra}
		\begin{aligned}
			\mathbb{E}\left[\sup _{g_{\pi} \in \mathcal{G}_{R,\pi}, \mathbb{E} g_{\pi}^{2} \leqslant r}\left|\mathbb{E} g_{\pi}-\frac{1}{n} \sum_{i=1}^{n} g_{\pi}\left(\bm z_{i}\right)\right|\right] & \leqslant c'_{q} d^{\frac{1}{2}} \max \left\{B^{\frac{2-q}{2+q}} n^{-\frac{2}{2+q}} (4BR)^{\frac{2q}{2+q}}, r^{\frac{1}{2}-\frac{q}{4}} n^{-\frac{1}{2}} (4BR)^{\frac{q}{2}}\right\} \\
			& \leqslant c'_{q} d^{\frac{1}{2}} (4BR)^{\frac{2q}{2+q}} \max \left\{B^{\frac{2-q}{2+q}} n^{-\frac{2}{2+q}}, r^{\frac{1}{2}-\frac{q}{4}} n^{-\frac{1}{2}}\right\}\,.
		\end{aligned}
	\end{equation}
	
\noindent	{\bf Step 4: Bound ${\tt S_{13}}$}\\
	According to Lemma~\ref{lemsubroot}, we take the right-hand of the above inequality~(\ref{eqslemra}) as the sub-root function $\psi(r)$. Then the solution $r^*$ to the equation $\psi(r)=r/D$ satisfies
	\begin{equation*}
		r^{*} \leq c'_{q} \max \left\{D^{\frac{4}{2+q}}, D B^{\frac{2-q}{2+q}}\right\} n^{-\frac{2}{2+q}} (4BR)^{\frac{2q}{2+q}} d^{\frac{1}{2}} \,.
	\end{equation*}
	Then all the conditions in Lemma~\ref{lemsubroot} are satisfied for the function set $\mathcal{G}_{R,\pi}$ with $\alpha = 1$, $Q := (B+CM) \left[4(B+4CM^2) \exp \left( -\frac{B}{4CM^2} \right) + 3B+CM \right] $,  $a:= c'_q R^q (4B)^q d^{\frac{1}{2}}$, and $\tau :=   (3B+CM) \left[ 3B + CM + 8 ( B + 4 CM^2 ) \exp\Big(-\frac{B}{4CM^2} \Big)  \right] $.
	Therefore, by Lemma~\ref{lemsubroot}, there exist a subset $Z_4$ of $Z^n$ with probability $1 - {\delta}/{4}$ such that 
	\begin{equation}\label{eqs13}
		\begin{split}
			{\tt S_{13}} & :=	\left[  	\widetilde{\mathcal{E}}\big(\pi_B (f_{\bm \theta^{\star}})\big) -  \widetilde{\mathcal{E}}(f_{\rho})   \right]  -  \left[  	\widetilde{\mathcal{E}}_{\bm z}\big(\pi_B (f_{\bm \theta^{\star}})\big) -  \widetilde{\mathcal{E}}_{\bm z}(f_{\rho})   \right] \\
			& \leqslant  \frac{1}{2}  \mathbb{E} [ g_{\pi} + \varOmega ] + c_q' \eta + \frac{18Q + 2\tau}{n} \log \frac{4}{\delta},  \quad \forall \bm z \in Z_4, f \in \mathcal{G}_R \,,
		\end{split}
	\end{equation}
	where 
	\begin{equation*}
		\eta=\max \left\{Q^{\frac{2-q}{2+q}}, \tau^{\frac{2-q}{2+q}}\right\}\left(\frac{a}{n}\right)^{\frac{2}{2+q}} \leqslant 16 B^2 R^{\frac{2q}{2+q}} \tau  \left[ c_q^{\frac{2}{2+q}} n^{-\frac{2}{2+q}} \right] d^{\frac{1}{2+q}} \,.
	\end{equation*}

	Combining the above three equations~(\ref{eqs11}), ~(\ref{eqs12}), ~(\ref{eqs13}) into \cref{YYS1} to estimate ${\tt S_1}$, for any $\bm z  \in Z' := Z_3 \cap Z_4$ and $f \in \mathcal{G}_R$, the following result holds with probability at least $1-\delta/2$
	\begin{equation*}
		\begin{split}
			&   \Big\{	\mathcal{E}\big[ \pi_B (f_{\bm \theta^{\star}}) \big] -  \mathcal{E}(f_{\rho})  \Big\} - \Big\{	\mathcal{E}_{\bm z}\big[ \pi_B (f_{\bm \theta^{\star}}) \big] -  \mathcal{E}_{\bm z}(f_{\rho})  \Big\} \leqslant \frac{1}{2}  \Big\{	\mathcal{E}\big[ \pi_B (f_{\bm \theta^{\star}}) \big] -  \mathcal{E}(f_{\rho})  \Big\}  \\
			& \quad 
			+ 4 (B + CM) (B + 4CM^2)  \exp \left( -\frac{B}{4CM^2} \right) + \tau \left( 16 B^2 R^{\frac{2q}{2+q}}  c_q' c_q^{\frac{2}{2+q}} d^{\frac{1}{2+q}} n^{-\frac{2}{2+q}} + \frac{20}{n} \log \frac{4}{\delta} \right)  \\
			& \quad  + 2 (B + CM) \left[ \frac{24CM^2}{n} \log \frac{4}{\delta} + 136CM^2  \exp \Big( -\frac{B}{4CM^2} \Big) \right]  \\
			& \leqslant \frac{1}{2} \left[	\mathcal{E}\big( \pi_B(f_{\bm \theta^{\star}}) \big) -  \mathcal{E}(f_{\rho}) \right] + 4(B+CM) (B+ 72CM^2) \exp \Big( -\frac{B}{4CM^2} \Big) \\
			& \quad  + 68(3B+CM) \left[3B + 49CM^2 + 8(B+4CM^2) \exp \Big( -\frac{B}{4CM^2} \Big) \right]  \frac{1}{n} \log \frac{4}{\delta} \\
			& + 16B(3B+CM) \left[3B + CM + 8(B+4CM^2) \exp \Big( -\frac{B}{4CM^2} \Big) \right] R^{\frac{2q}{q+2}} c_q' c_q^{\frac{2}{2+q}} d^{\frac{1}{2+q}} n^{-\frac{2}{2+q}} \log \frac{4}{\delta}  \,,
		\end{split}
	\end{equation*}
	where we conclude the proof by taking
	$\widetilde{C} := B(B + CM^2) \left[ \exp\Big(-\frac{B}{4CM^2} \Big) + 2CM^2 + \| \bm x \|_{\infty} \right] $ independent of $n, \delta$.
	
	Finally, combining the bounds for the output error in Proposition~\ref{propoutput}, and the sample error including ${\tt S_2}$ in Proposition~\ref{props2} and ${\tt S_1}$ in Proposition~\ref{props1}, and the regularization error, for any $0< \delta<1$, with probability at least $1 - \delta$, we have (for $R \geqslant M \geqslant 1$)
	\begin{equation*}
	\begin{split}
		\mathcal{E} ( \pi_B ( f_{\bm \theta^{(T)}} )  )	- \mathcal{E} ( f_{\rho} )  & \leqslant  C_1 R^{\frac{2q}{q+2}} d^{\frac{1}{2+q}} n^{-\frac{2}{2+q}} \log \frac{4}{\delta} + C_2 \exp\Big(-\frac{B}{4CM^2} \Big) \\
		& + \frac{C_3}{n} \log^2 \frac{4}{\delta} + 38 CM \lambda + \frac{114R^2}{m} \,,
		\end{split}
	\end{equation*}
	where $C_1 = {\tt poly}(B, C, M, \| \bm x \|_{\infty}, c_q, c_q')$, $C_2 = {\tt poly} (B, C, M)$, and $C_3 = {\tt poly} (C, M, \| \bm x \|_{\infty})$ are some constants independent of $n$, $\delta$, and $R$. The {\tt poly} order on $B$ is at most $3$, which means that $C_2 \exp\Big(-\frac{B}{4CM^2} \Big)$ can converge to zero in an exponential order for a large $B$.
	Finally we finish the proof by taking $q=\frac{2d}{d+2}$.
	
\end{proof}

\section{Auxiliary lemmas}
\label{app:auxlemma}

Here we present some useful lemmas that our proof is needed.
Let us recall some useful properties of Gaussian processes by the following two lemmas.

\begin{lemma}\citep{kamath2015bounds}
	\label{lem:gaussians}
	Let $\{ X_i\}_{i=1}^n$ be i.i.d $\mathcal{N}(0, 1)$ random variables, then consider the random variable $Z:= \max \limits_{i = 1,2, \cdots, n} X_i$, we have
	\[
	\mathbb{E} [Z] \geqslant \frac{1}{\sqrt{\pi \log 2}}\sqrt{\log n}\,.
	\]
\end{lemma}

\begin{lemma}\cite[Theorem 5.27]{wainwright2019high} (Sudakov-Fernique)\label{lem:sudakov}
	Let $\{ X_t\}_{t=1}^T$ and $\{ Y_t\}_{t=1}^T$ be a pair of zero-mean Gaussian vectors, if $\mathbb{E} (X_t - X_s)^2 \geqslant \mathbb{E} (Y_t - Y_s)^2 $ for all $t, s \in T$, then, we have
	\[
	\mathbb{E}\left[\max_{t =1,2,\cdots, T} X_t\right] \geqslant \mathbb{E}\left[\max_{t =1,2,\cdots, T} Y_t\right]\,.
	\]
\end{lemma}

We also need the following two concentration inequalities for random variables.

\begin{lemma} \citep{bennett1962probability}\label{lemaunbounded}
	Let $A_1, A_2, \cdots, A_n$ be independent random variables with $\mathbb{E}(A_i) =0$. If there exists some constants $M, v > 0$ such that  $\mathbb{E}|A_i|^{p} \leqslant \frac{1}{2} p ! M^{p - 2} v$ holds for $2 \leqslant p \in \mathbb{N}$, then
	\begin{equation}\label{leamm4eq}
		\mathop{\rm Prob} \left\{ \sum_{i=1}^m A_i  \geqslant \varepsilon  \right\}  \leqslant 
		\exp \!   \left\{ \frac{-\varepsilon^{2}}{2(mv + M \varepsilon)} \right\} \,, \forall \varepsilon > 0\,.
	\end{equation}
\end{lemma}

\begin{lemma} \cite[Proposition 6]{wu2007multi}\label{lemsubroot}
	Let $\mathcal{F}$ be a set of measurable functions on $Z$, assume that there exists two positive constants $Q, \tau$ and $\alpha \in [0,1]$ such that $\| f \|_{\infty} \leqslant Q$ and $\mathbb{E}[f^2] \leqslant \tau \mathbb{E} [f^{\alpha}]$ for every $f \in \mathcal{F}$. If for some $a > 0$ and $0 < q < 2$, 
	\begin{equation}\label{covnum1}
		\sup_{n \in \mathbb{N}} \sup_{\bm z \in Z^m} \log \mathcal{N}_2 (\mathcal{F}, \varepsilon) \leqslant a \varepsilon^{-q}\,, \quad \forall \varepsilon > 0\,.
	\end{equation}
	Then there exists a constant $c_q'$ only depending on $q$ such that for any $t > 0$, the following proposition holds with probability at least $1 - e^{-t}$
	\begin{equation*}
		\mathbb{E} f-\frac{1}{n} \sum_{i=1}^{n} f\left(\bm z_{i}\right) \leqslant \frac{1}{2} \eta^{1-\alpha}(\mathbb{E} f)^{\alpha}+c_{q}^{\prime} \eta + 2\left(\frac{\tau t}{n}\right)^{\frac{1}{2-\alpha}}+\frac{18 Q t}{n} \quad \forall f \in \mathcal{F} \,,
	\end{equation*}
	where 
	\begin{equation}
		\eta:=\max \left\{\tau^{\frac{2-q}{4-2 \alpha+q \alpha}}\left(\frac{a}{n}\right)^{\frac{2}{4-2 \alpha+q \alpha}}, \quad Q^{\frac{2-q}{2+q}}\left(\frac{a}{n}\right)^{\frac{2}{2+q}}\right\} \,.
	\end{equation}
\end{lemma}

\vskip 0.2in
\bibliography{refs}

\end{document}